\documentclass[english,10pt,conference,letterpaper]{IEEEconf}
\usepackage[T1]{fontenc}
\usepackage[latin9]{inputenc}
\usepackage{verbatim}
\usepackage{prettyref}
\usepackage{refstyle}
\usepackage{float}
\usepackage{textcomp}
\usepackage{amsmath}
\usepackage{amsthm}
\usepackage{amssymb}
\usepackage{graphicx}

\makeatletter

\AtBeginDocument{\providecommand\exaref[1]{\ref{exa:#1}}}
\AtBeginDocument{\providecommand\figref[1]{\ref{fig:#1}}}
\AtBeginDocument{\providecommand\defref[1]{\ref{def:#1}}}
\RS@ifundefined{subsecref}
  {\newref{subsec}{name = \RSsectxt}}
  {}
\RS@ifundefined{thmref}
  {\def\RSthmtxt{theorem~}\newref{thm}{name = \RSthmtxt}}
  {}
\RS@ifundefined{lemref}
  {\def\RSlemtxt{lemma~}\newref{lem}{name = \RSlemtxt}}
  {}

\theoremstyle{plain}
\newtheorem{thm}{\protect\theoremname}
\theoremstyle{definition}
\newtheorem{defn}[thm]{\protect\definitionname}
\theoremstyle{definition}
\newtheorem{example}[thm]{\protect\examplename}
\theoremstyle{plain}
\newtheorem{lem}[thm]{\protect\lemmaname}
\theoremstyle{plain}
\newtheorem{prop}[thm]{\protect\propositionname}
\theoremstyle{remark}
\newtheorem{rem}[thm]{\protect\remarkname}

\usepackage{amssymb}
\usepackage{tikz}
\usetikzlibrary{decorations.pathreplacing}
\usepackage{xspace}

\newcommand{\les}{\lesssim}
\newcommand{\less}{<} 
\newcommand{\rleq}{\leq}

\newcommand{\TODO}[1]{{\color{blue}\textbf{TODO}: #1}}

\newcommand{\trajectories}{realizations\xspace}

\newcommand{\posReals}{\mathbb{R}_{+}}
\newcommand{\xxx}{{\color{red}xxx}\xspace}

\usepackage{prettyref}
\newrefformat{prob}{Problem~\ref{#1}}
\newrefformat{sub}{Section~\ref{#1}}
\newrefformat{prop}{Proposition~\ref{#1}}
\newrefformat{def}{Definition~\ref{#1}}
\newrefformat{app}{Appendix~\ref{#1}}
\newrefformat{alg}{Algorithm~\ref{#1}}
\newrefformat{cor}{Corollary~\ref{#1}}
\newrefformat{thm}{Theorem~\ref{#1}}
\newrefformat{tab}{Table~\ref{#1}}
\newrefformat{fig}{Fig.~\ref{#1}}
\newrefformat{sec}{Section~\ref{#1}}
\newrefformat{def}{Def.~\ref{#1}}
\newrefformat{line}{line~\ref{#1}}
\newrefformat{code}{Listing~\ref{#1}}
\newrefformat{exa}{Example~\ref{#1}}

\newcommand{\exaref}[1]{Example~\ref{exa:#1}}
\renewcommand{\figref}[1]{Fig.~\ref{fig:#1}}

\IEEEoverridecommandlockouts
\newcommand{\mythanks}{\thanks{The authors are with nuTonomy, an Aptiv company (Boston, MA, Zurich, and Singapore). Please address correspondence to andrea@nutonomy.com.}}

\usepackage[font={footnotesize}]{caption}

\let\labelindent\relax
\usepackage{enumitem}
\setlist[enumerate]{itemsep=0mm,leftmargin=4mm,topsep=0mm}

\usepackage[noadjust]{cite}

\usepackage{ifxetex}
\ifxetex
\usepackage[protrusion=true,tracking=false,kerning=false,spacing=false]{microtype}
\else

\usepackage[tracking=false,kerning=true,spacing=true]{microtype}

\fi

\usepackage{breqn}

\@ifundefined{showcaptionsetup}{}{%
 \PassOptionsToPackage{caption=false}{subfig}}
\usepackage{subfig}
\makeatother

\usepackage{babel}
\providecommand{\definitionname}{Definition}
\providecommand{\examplename}{Example}
\providecommand{\lemmaname}{Lemma}
\providecommand{\propositionname}{Proposition}
\providecommand{\remarkname}{Remark}
\providecommand{\theoremname}{Theorem}

\begin{document}

\title{Liability, Ethics, and Culture-Aware Behavior Specification using
Rulebooks}

\author{Andrea Censi, Konstantin Slutsky, Tichakorn Wongpiromsarn, \\
Dmitry Yershov, Scott Pendleton, James Fu, Emilio Frazzoli\mythanks}
\maketitle
\begin{abstract}
The behavior of self-driving cars must be compatible with an enormous
set of conflicting and ambiguous objectives, from law, from ethics,
from the local culture, and so on. This paper describes a new way
to conveniently define the desired behavior for autonomous agents,
which we use on the self-driving cars developed at nuTonomy, an Aptiv
company.

We define a ``rulebook'' as a pre-ordered set of ``rules'', each
akin to a violation metric on the possible outcomes (``realizations'').
The rules are ordered by priority. The semantics of a rulebook imposes
a pre-order on the set of realizations. We study the compositional
properties of the rulebooks, and we derive which operations we can
allow on the rulebooks to preserve previously-introduced constraints. 

While we demonstrate the application of these techniques in the self-driving
domain, the methods are domain-independent.
\end{abstract}

\section{Introduction}

One of the challenges in developing self-driving cars is simply \emph{defining}
\emph{what the car is supposed to do}. The behavior specification
for a self-driving car comes from numerous sources, including not
only the vaguely specified ``rules of the road'', but also implementation
limitations (for example, the speed might be limited due to the available
computation for perception), and numerous other soft constraints,
such as the need of ``appearing natural'', or to be compatible with
the local driving culture (any reader who never lived in Boston will
be surprised to discover what is the ``Massachusetts left''). As
self-driving cars are potentially life-endangering, also moral and
ethical factors play a role~\cite{Arkin2008,Thornton17,students18}.
For a self-driving car, the ``trolley problems''~\cite{Foot1967-FOOTPO-2,Thomson1985,moralmachine}
are not idle philosophical speculations, but something to solve in
a split second. As of now, there does not exist a formalism that allows
to incorporate all these factors in one specification, which can be
precise enough to be taken as regulation for what self-driving cars
designers must implement. 

Formal methods have been applied to specify and verify properties
of complex systems. The main focus has been to provide a proof that
the system satisfies a given specification, expressed in a formal
language. In particular, specifications written in temporal logics
have been studied extensively~\cite{Fainekos05,Kloetzer08,Tabuada04lineartime,karaman08cdc,LOTM13,Webster:2011:FMC:2041619.2041644}.
In self-driving cars, often times, not all the rules can be satisfied
simultaneously. Although there are formalisms that allow specifying
the degree of satisfaction of each rule, e.g., based on fuzzy logic
or some measurable probability~\cite{MorseAELR17,Cizelj13}, as of
now, there does not exist a formalism that allows to incorporate different
factors needed to be considered for self-driving cars with a precise
hierarchy in one specification. 

\begin{figure}
\begin{centering}
\subfloat[Autonomy is about making\newline\hspace{\linewidth}\rule{3mm}{0pt}
the right choices\emph{.}]{\begin{centering}
\includegraphics[height=2.2cm]{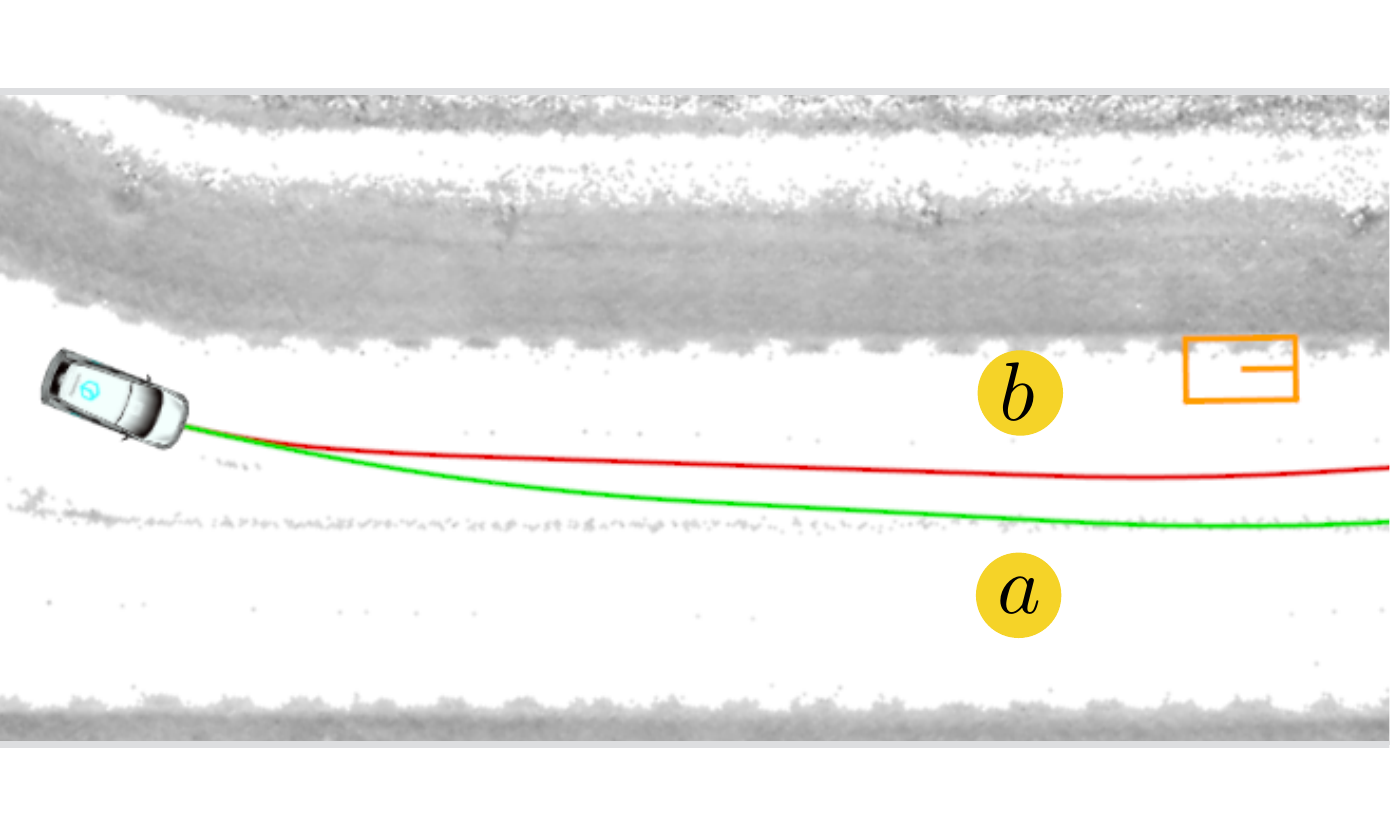}
\par\end{centering}
}\rule{3pt}{0pt}\subfloat[\label{fig:Rulebook-and-induced}Rulebook and induced order\newline\hspace{\linewidth}\rule{3mm}{0pt}
on realizations (outcomes).]{\begin{centering}
\includegraphics[height=2.3cm]{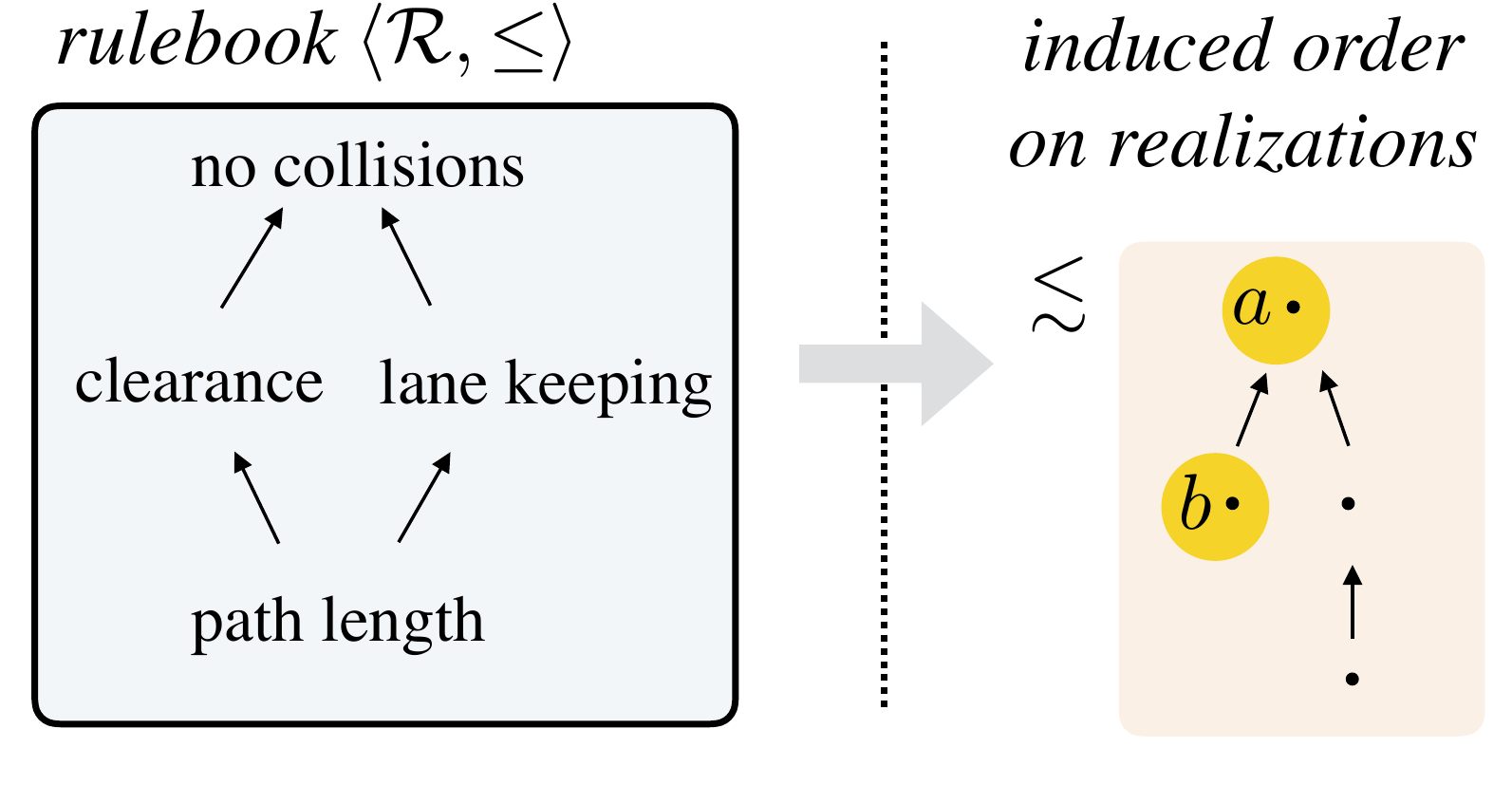}
\par\end{centering}
}
\par\end{centering}
\subfloat[\label{fig:Iterative-specification-refinmen}Rulebook manipulation
operations refine the specification]{\begin{centering}
\includegraphics[width=1\columnwidth]{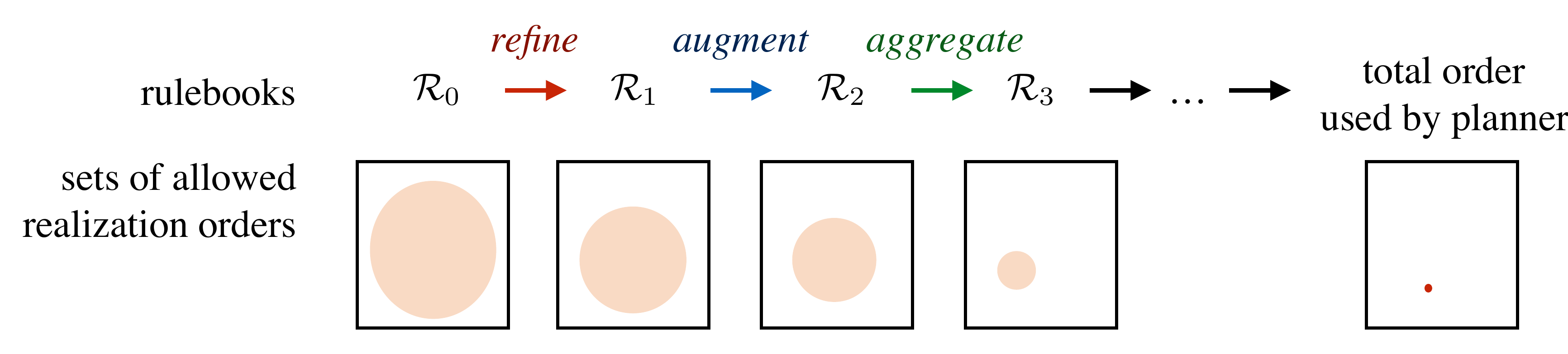}
\par\end{centering}
}\caption{The rulebooks formalism allows to specify the desired behavior for
an autonomous agent by using a pre-ordered set of rules that induce
a pre-order on the allowed outcomes. The rulebooks can be refined
by a series of manipulation operations.}
\end{figure}

In this paper we describe a formalism called ``rulebooks'', which
we use to specify the desired behavior of the self-driving cars developed
at nuTonomy\footnote{Please note that the functionality described is not necessarily representative
of current and future products by nuTonomy, Aptiv and their partners.
The scenarios discussed are simplified for the purposes of exposition.
The specification examples discussed are illustrative of the philosophy
but not the precise specification we use. The methodology described
does not represent the full development process; in particular we
gloss over the extensive verification and validation processes that
are needed for safety-critical rules.}. While the formalism can be applied to any system, it is particularly
well-suited to handle behavior specification for embodied agents in
an environment where many, possibly conflicting rules need to be honored.
We define a ``rulebook'' as a set of ``rules'' (\prettyref{fig:Rulebook-and-induced}),
each akin to a violation metric on the possible outcomes. The rules
can be defined analytically, using formalisms such as LTL~\cite{doi:10.1146/annurev-control-060117-104838}
or even deontic logic~\cite{Hoven}, or the violation functions can
be learned from data, using inverse reinforcement learning~\cite{2012-cioc},
or any technique that allows to measure deviation from a model. In
the driving domain, the rules can derive from traffic laws, from common
sense, from ethical considerations, etc.

The rules in a rulebook are hierarchically ordered to describe their
relative priority, but, following the maxim ``\emph{good specifications
specify little}'', the semantics of a rulebook imposes a \emph{pre-order}
on the set of outcomes, which means that the implementations are left
with considerable freedom of action. The rulebooks formalism is ``user-oriented'':
we define a set of intuitive operations that can be used to iteratively
refine the behavior specification. For example, one might define an
``international rulebook'' for rules that are valid everywhere,
and then have region-specific rulebooks for local rules, such as which
side the car should drive on. 

While the rulebooks offer formidable generality in describing behavior,
at the same time, when coupled with graph-based motion planning, the
rulebooks allow a \emph{systematic, simple, and scalable solution
to planning }for a self-driving car:
\begin{enumerate}
\item Liability-, ethics-, culture-derived constraints are formulated as
rules (preferences over trajectories), either manually or in a data-driven
fashion, together with the rules of the road and the usual geometric
constraints for motion planning.
\item Priorities between conflicting rules are established as a rulebook
(ideally, by nation-wide regulations based on public discourse);
\item Developers customize the behavior by resolving ambiguities in the
rulebook until a total order is obtained;
\item Graph-based motion planning, in particular variations of minimum-violation
motion planning~\cite{6760374,Tumova:2013:LCS:2461328.2461330,Tumova2013,Vasile2017,US9645577B1},
allow to generate the trajectories that maximally respect the rules
in the rulebooks.
\end{enumerate}
In a nutshell, the above is how the nuTonomy cars work. The topic
of \emph{efficiently} planning with rulebooks is beyond the goals
of this paper; here, we focus on the use of the rulebooks as a specification,
treating the planning process as a black box.%

\section{Rulebooks definition}

\subsubsection{Realizations}

Our goal is to define the desired agent \emph{behavior}. Here, we
use the word ``behavior'' in the sense of Willems~\cite{Polderman:1997:IMS:265883}
(and not in the sense of ``behavior-based robotics''~\cite{Arkin:1998:BR:521898,MataricM08}),
to mean that what we want to prescribe is what we can measure objectively,
that is, the externally observable actions and outcomes in the world,
rather than the internal states of the agent or any implementation
details. %
Therefore, we define preference relations on a set of possible outcomes,
which we call the set of \emph{realizations}~$\Xi.$ For a self-driving
car, a realization~$x\in\Xi$ is a world trajectory, which includes
the trajectory of all agents in the environment. 

We use \emph{no} concept of infeasibility. Sometimes the possible
outcomes are all catastrophically bad; yet, an embodied agent must
keep calm and choose the least catastrophic option.

\subsubsection{Rules}

Our ``atom'' of behavioral specification is the ``rule''. In our
approach, a rule is simply a scoring function, or ``violation metric'',
on the realizations. 
\begin{defn}[Rule]
\label{def:rule}Given a set of realizations~$\Xi$, a \emph{rule}
on~$\Xi$ is a function~$r:\Xi\rightarrow\posReals.$
\end{defn}
The function~$r$ measures the degree of violation of its argument.
If~$r(x)<r(y)$, then the realization~$y$ violates the rule~$r$
to a greater extent than does~$x$. In particular, $r(x)=0$ indicates
that a realization~$x$ is fully compliant with the rule. 

Any scalar function will do. The definition of the violation metric
might be analytical, ``from first principles'', or be the result
of a learning process.

In general, the rulebooks philosophy is to pay particular attention
about specifying what we ought to do \emph{when the rule has to be
violated}, as described in the following examples.
\begin{example}[Speed limit]
A naïve rule that is meant to capture a speed limit of 45 km/h could
be defined as:
\[
r(x)=\begin{cases}
0, & \text{if the car's speed is always below 45\,km/h,}\\
1, & \text{otherwise.}
\end{cases}
\]
However, this discrete penalty function is not very useful in practice.
The rulebooks philosophy is to assume that rules might need to be
violated for a greater cause. In this case, it is advisable to define
a penalty such as: 
\[
r(x)=\text{interval for which the car was above 45 km/h.}
\]
The effect of this will be that the car will try to stay below the
speed limit, but if it cannot, it will minimize the time spent violating
the limit. Alternatively, one can penalize also the magnitude of the
speed violation:
\[
r'(x)=r(x)\times(v_{\max}-\text{45 km/h}).
\]
\end{example}
\begin{example}[Minimizing harm]
It is easy enough to write a constraint describing the fact that
we do not want any collision; but, assuming that a collision with
a human is unavoidable given the circumstances, what should the car
do? In this case, it would be advisable to define the violation function
as:
\[
r(x)=\text{kinetic energy transferred to human bodies,}
\]
so that the car will try to avoid collisions, but, if a collision
is inevitable, it will try to reduce the speed as much as possible.
\end{example}

\subsubsection{Rulebooks}

A rulebook~$\mathcal{R}$ is a \emph{pre-ordered} set of rules. We
will use~$\mathcal{R}$ both for the rulebook and for its underlying
set of rules.
\begin{defn}[Rulebook]
A \emph{rulebook} is a tuple $\left\langle \mathcal{R},\rleq\right\rangle $,
where~$\mathcal{R}$ is a finite set of rules and~$\rleq$ is a
preorder on~$\mathcal{R}$.
\end{defn}
\begin{figure}[H]
\begin{centering}
\includegraphics[width=4.5cm]{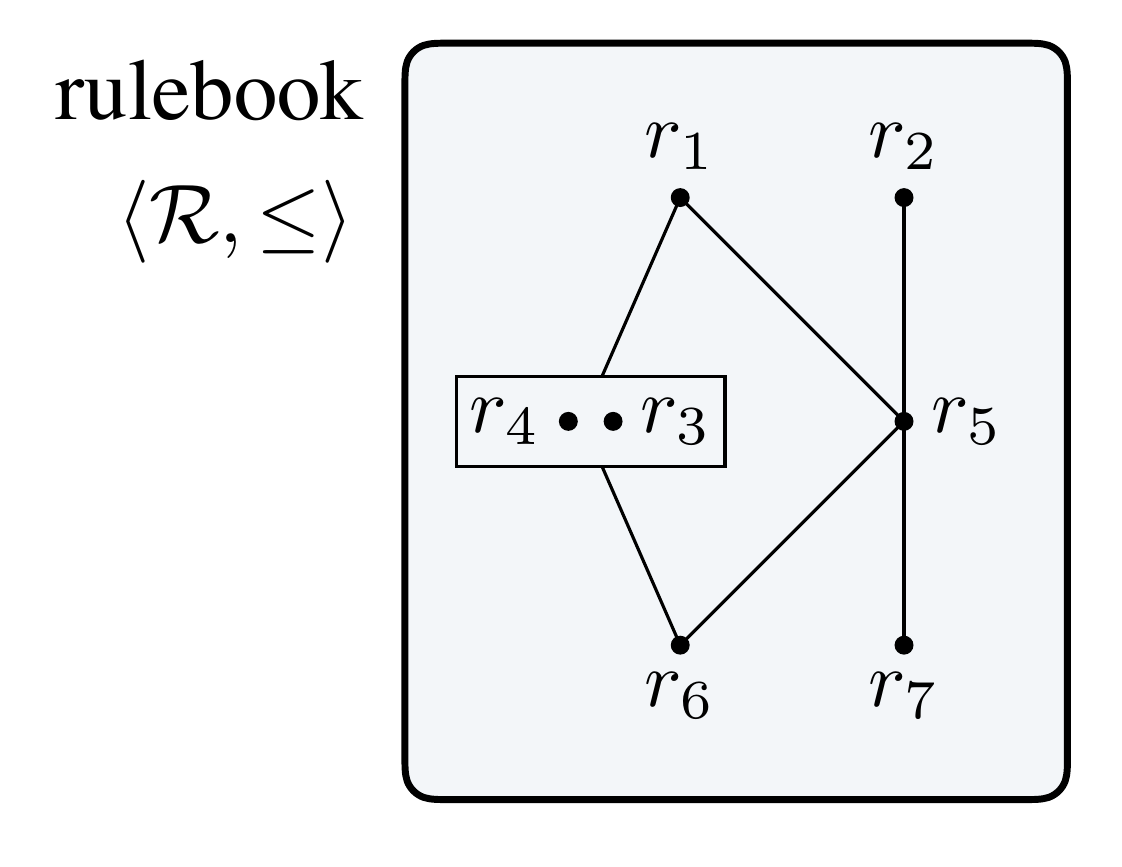}
\par\end{centering}
\caption{\label{fig:rulebook-graph}Graphical representation of a rulebook.
Rules are ordered vertically with the most important rules being at
the top.}
\end{figure}

Being a preorder, any rulebook may be represented as a directed graph,
in which each node is a rule, and an edge between two rules~$r_{1}\rightarrow r_{2}$
means that~$r_{1}\rleq r_{2}$, i.e., the rule~$r_{2}$ has higher
rank. \prettyref{fig:rulebook-graph} gives an example of a rulebook
with $7$ rules. In this example, rules~$r_{1}$ and~$r_{2}$ are
incomparable, but both are greater than $r_{5}$. Rules~$r_{3}$
and~$r_{4}$ are of the same rank, meaning $r_{3}\le r_{4}$ \emph{and}
$r_{4}\le r_{3}$, and both are smaller than~$r_{1}$, greater than~$r_{6}$
and incomparable to~$r_{5}$, $r_{2}$, or $r_{7}$. 

Just like it might be convenient to learn some of the non-safety-critical
rules from data, it is possible to learn some of the priorities from
data as well. (See~\cite{modugno:hal-01273409,DBLP:journals/corr/SilverioCRC17}
for a similar concept in a different context.)

\subsubsection{Induced pre-order on realizations}

We now formally define the semantics of a rulebook as specifying a
pre-order on realizations. Because a rulebook is defined as a \emph{pre-ordered}
set of rules, not all the relative priorities among different rules
are specified. We will see that this means that a rulebook can be
used as a very flexible \emph{partial} specification.

Given a rulebook $\left\langle \mathcal{R},\rleq\right\rangle $,
our intention is to preorder all realizations such that~$x\les y$
can be interpreted as $x$ being ``at least as good as''~$y$,
i.e., the degree of violation of the rules by~$x$ is at most as
much as that of~$y$.
\begin{defn}[Pre-order $\les$ and strict version~$\less\:$]
\label{def:preorder-on-trajectories} Given a rulebook~$\left\langle \mathcal{R},\rleq\right\rangle $
and two realizations $x,y\in\Xi$, we say that $x\les y$ if for any
rule~$r\in\mathcal{R}$ satisfying~$r(y)<r(x)$ there exists a rule
$r'>r$ such that~$r'(x)<r'(y)$. We denote by~$\less\:$ the strict
version of~$\les$.
\end{defn}

\begin{lem}
\label{lem:transitivity} Let $\left\langle \mathcal{R},\rleq\right\rangle $
be a rulebook, let $x,y,z\in\Xi$ be realizations such that $x\les y$,
$y\les z$, and let $r\in\mathcal{R}$ be a rule. If either $r(x)\ne r(y)$
or $r(y)\ne r(z)$, then there exists a rule $r'\ge r$ such that
$r'(x)<r'(z)$. 
\end{lem}
\begin{IEEEproof}
We give a proof for $r(x)\ne r(y)$; the case $r(y)\ne r(z)$ is analogous.
If $r(x)>r(y)$, then $x\les y$ guarantees existence of $r_{0}>r$
such that $r_{0}(x)<r_{0}(y)$. If $r(x)<r(y)$ to begin with, then
we may set $r_{0}=r$, and in either case we get $r_{0}\ge r$ such
that $r_{0}(x)<r_{0}(y)$. We are done if $r_{0}(y)\le r_{0}(z)$,
as one can take $r'=r_{0}$. Otherwise, $y\les z$ implies existence
of $r_{1}>r_{0}$ such that $r_{1}(y)<r_{1}(z)$. Again, we are done
if $r_{1}(x)\le r_{1}(y)$, and if not, there has to be some rule
$r_{2}>r_{1}$ such that $r_{2}(x)<r_{2}(y)$. Continuing in the same
fashion, one builds an increasing chain $r\le r_{0}<r_{1}<r_{2}<\cdots.$
Since $\mathcal{R}$ is assumed to be finite, the chain has to stop,
which is possible only if $r_{n}(x)<r_{n}(z)$ for some $n$. 
\end{IEEEproof}
\begin{prop}
Let $\left\langle \mathcal{R},\rleq\right\rangle $ be a rulebook,
and let $x,y,z\in\Xi$ be realizations. 
\begin{enumerate}
\item The relation $\les$ on realizations is a preorder. 
\item Two realizations $x$ and $y$ are equivalent if and only if $r(x)=r(y)$
for all rules $r\in\mathcal{R}$.
\end{enumerate}
\end{prop}
\begin{IEEEproof}
1) It is clear that $\les$ is reflexive, so we only need to check
transitivity. Suppose $x\les y$, $y\les z$, and let $r\in\mathcal{R}$
be such that $r(x)>r(z)$. Clearly either $r(y)\ne r(x)$ or $r(z)\ne r(y)$,
so Lemma \ref{lem:transitivity} applies, producing $r'>r$ such that
$r'(x)<r'(z)$, hence $x\les z$ as claimed. 

2) Suppose towards a contradiction there are some realizations satisfying
$x\les y$ and $y\les x$, yet $r(x)\ne r(y)$ for some $r\in\mathcal{R}$.
Without loss of generality, let us assume that $r(x)<r(y)$. Since
we have $x\les y\les x$, Lemma~\ref{lem:transitivity} produces
some $r'\ge r$ such that $r'(x)<r'(x)$, which is absurd. The other
direction ($\forall r,r(x)=r(y)\implies x\sim y)$ is obvious.
\end{IEEEproof}
\begin{rem}
In the special case in which the rulebook is a linear order, the induced
order on realizations is the \emph{lexicographic order} used in the
literature in minimum-violation planning.
\end{rem}

\section{Examples in the driving domain\label{sec:hierarchy}}

In this section, we give a few examples of the types of rules that
are useful in the driving domain. Rather than describing the full
complexity of our production rules, which address subtle nuances of
behavior and idiosyncrasies and corner cases of traffic laws, we prefer
to give a few synthetic examples of rulebooks and rulebooks refinement.

\begin{example}[Safety vs. infractions]
Consider the scenario in~\prettyref{fig:collision}. A vehicle is
faced with an obstacle in front, and is given a choice between two
trajectories~$a$ and~$b$. Suppose the initial speed of the vehicle
is sufficiently high, and there is no time to stop, so collision is
unavoidable if $a$ is chosen. Trajectory~$b$, however, is collision
free, but it violates a different rule, since it intersects a double
solid line. 

The rulebooks take on this situation is the following. A rule ``not
to collide with other objects'' will have a higher priority than
the rule of not crossing the double line~(\prettyref{fig:no-coll}).
With this rulebook, the trajectory~$b$ will be chosen to avoid the
collision.
\end{example}
\begin{figure}[h]
\centering{}\subfloat[]{\begin{centering}
\includegraphics[width=0.6\columnwidth]{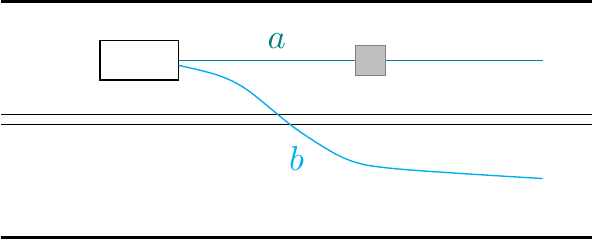}
\par\end{centering}
} \subfloat[\label{fig:no-coll}]{\begin{centering}
\includegraphics[width=0.3\columnwidth]{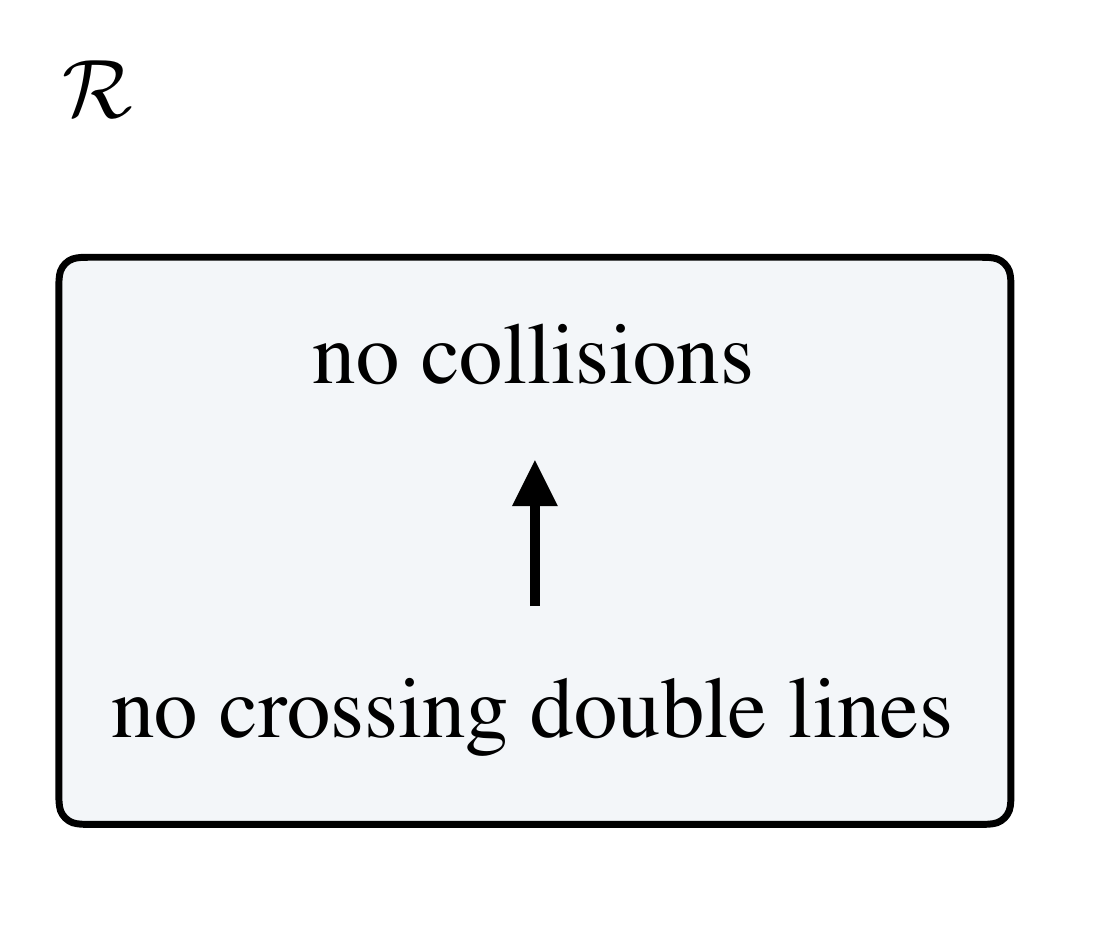}
\par\end{centering}
}\caption{\label{fig:collision}The rulebook allows the agent to cross the double
white line to avoid a collision. (This assumes that there are no other
agents outside the frame that might trigger the ``no-collision''
rule.)}
\end{figure}

\begin{figure}[h]
\centering{}\subfloat[]{\begin{centering}
\includegraphics[width=0.6\columnwidth]{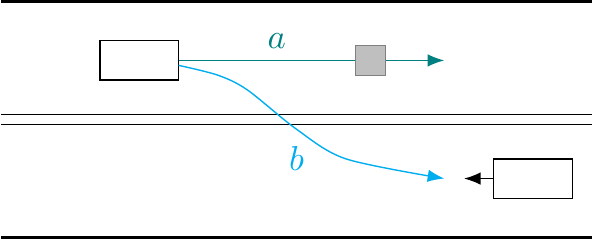}
\par\end{centering}
} \subfloat[\label{fig:collision-choice-rb}]{\begin{centering}
\includegraphics[width=0.3\columnwidth]{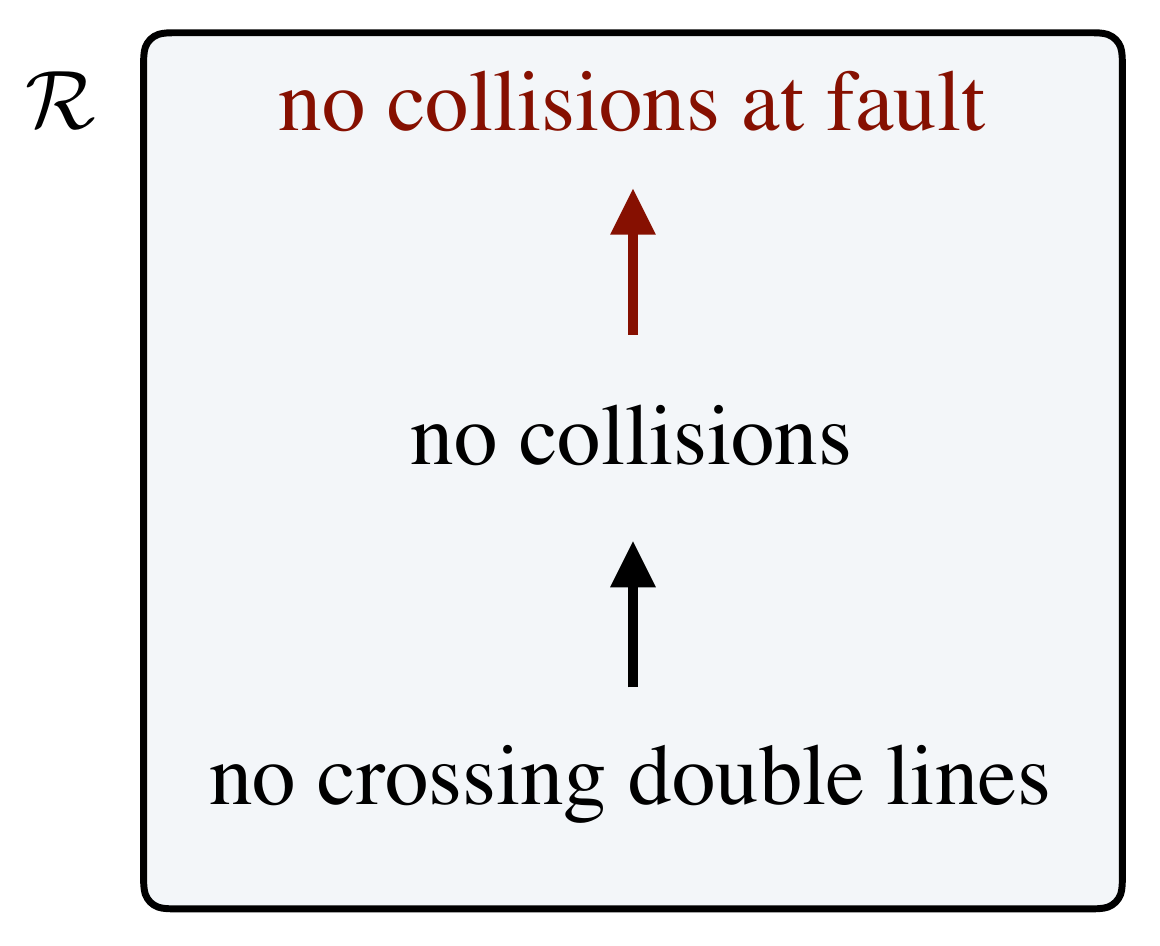}
\par\end{centering}
}\caption{\label{fig:collision-choice}The rulebook instructs the agent to collide
with the object on its lane, rather than provoking an accident, for
which it would be at fault.}
\end{figure}

\begin{example}[Liability-aware specification]
\label{exa:liability}Let's change the situation slightly by assuming
that trajectory $b$ is also in collision, but with a different agent
\textemdash{} an oncoming vehicle on the opposite lane. %
Under these assumptions, we may be interested in choosing the outcome
where the ego vehicle is \textit{not} at fault for the collision.

This behavior specification can be achieved by the rulebook of~\prettyref{fig:collision-choice-rb},
having two collision rules, where one evaluates the degree of collision,
where the ego vehicle is at fault, and the other evaluates collisions
caused by third-party, which is below the former in the rulebooks
hierarchy. This will force the ego vehicle to prefer trajectory $a$
over $b$.

This example fully captures the concept of the ``responsibility-sensitive
safety'' model described in~\cite{RSS17}.
\end{example}
\begin{example}[Partial priorities specification]
\label{exa:partial}Consider the scenario depicted in~\prettyref{fig:avoidance-example},
where the vehicle encounters an obstacle along its route. For simplicity,
we focus on four discrete representative trajectories, called $a,b,c,d$.
A minimal rulebook that allows to deal with this situation would contain
at least four rules, detailed below. For simplicity, we write the
violation metrics as binary variables having value $0$ or $1$ on
the test trajectories, while in practice these would be continuous
functions.%

\noindent 1) Rule $\beta$ - \textbf{Blockage}, attaining value $1$
if the trajectory is blocked by an obstacle, and $0$ otherwise: 
\[
\beta(x)=\begin{cases}
0, & \textrm{for }x=b,c,d;\\
1, & \textrm{for }x=a.
\end{cases}
\]

\noindent 2) Rule $\lambda$ - \textbf{Lane Keeping}, $1$ iff the
trajectory intersects the lane boundary:
\[
\lambda(x)=\begin{cases}
0, & \textrm{for }x=a,b;\\
1, & \textrm{for }x=c,d.
\end{cases}
\]

\noindent 3) Rule $\kappa$ - \textbf{Obstacle clearance}, $1$ iff
the trajectory comes closer to an obstacle than some threshold~$C_{0}$:
\[
\kappa(x)=\begin{cases}
0, & \textrm{for }x=c,d;\\
1, & \textrm{for }x=a,b.
\end{cases}
\]

\begin{rem}[Learning while preserving safety]
\noindent While parameters such as the minimum clearance from an
obstacle~$C_{0}$ can be specified manually, in practice, given an
adequate data analytics infrastructure, they are great candidates
to be \emph{learned} from the data. This allows the car to adapt the
behavior to the local driving culture. By still having the safety-preserving
rules at the top of the hierarchy, the rulebooks allow the system
to be adaptive without ever compromising safety, not even with adversarial
data. (See~\cite{Aswani2013} for a similar principle in a different
context.)
\end{rem}
\noindent 4) Rule $\alpha$ - \textbf{Path length}, whose value is
the length of the trajectory: 
\[
\alpha(a)<\alpha(b)<\alpha(c)<\alpha(d).
\]

Out of these rules, we can make different rulebooks by choosing different
priorities. For example, defining the rulebook~$\mathcal{R}$ with
ordering $\alpha<\kappa<\beta$ and $\alpha<\lambda<\beta$, depicted
in~\prettyref{fig:rulebook-example-order}, the following order on
trajectories is imposed: $b<a$ and $c<d<a$. Note that~$b$ is not
comparable with either~$d$ or~$c$. This is an important feature
of a \emph{partial specification:} we leave freedom to the implementation
to choose the details of the behavior that we do not care about.
\end{example}
\begin{figure}[h]
\begin{centering}
\subfloat[\label{fig:avoidance-example}Trajectories available to a vehicle
before an avoidance maneuver.]{\begin{centering}
\includegraphics[width=0.85\columnwidth]{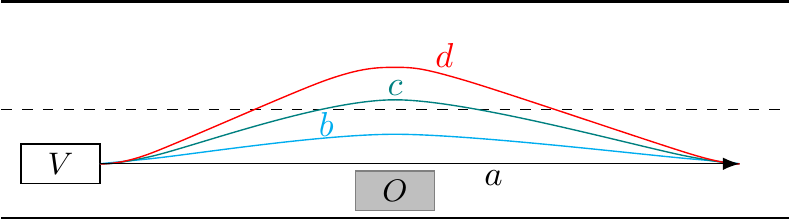} 
\par\end{centering}
\centering{}}
\par\end{centering}
\begin{centering}
\subfloat[\label{fig:rulebook-example-order}Rulebook hierarchy and induced
hierarchy on realizations order.]{\begin{centering}
\includegraphics[width=1\columnwidth]{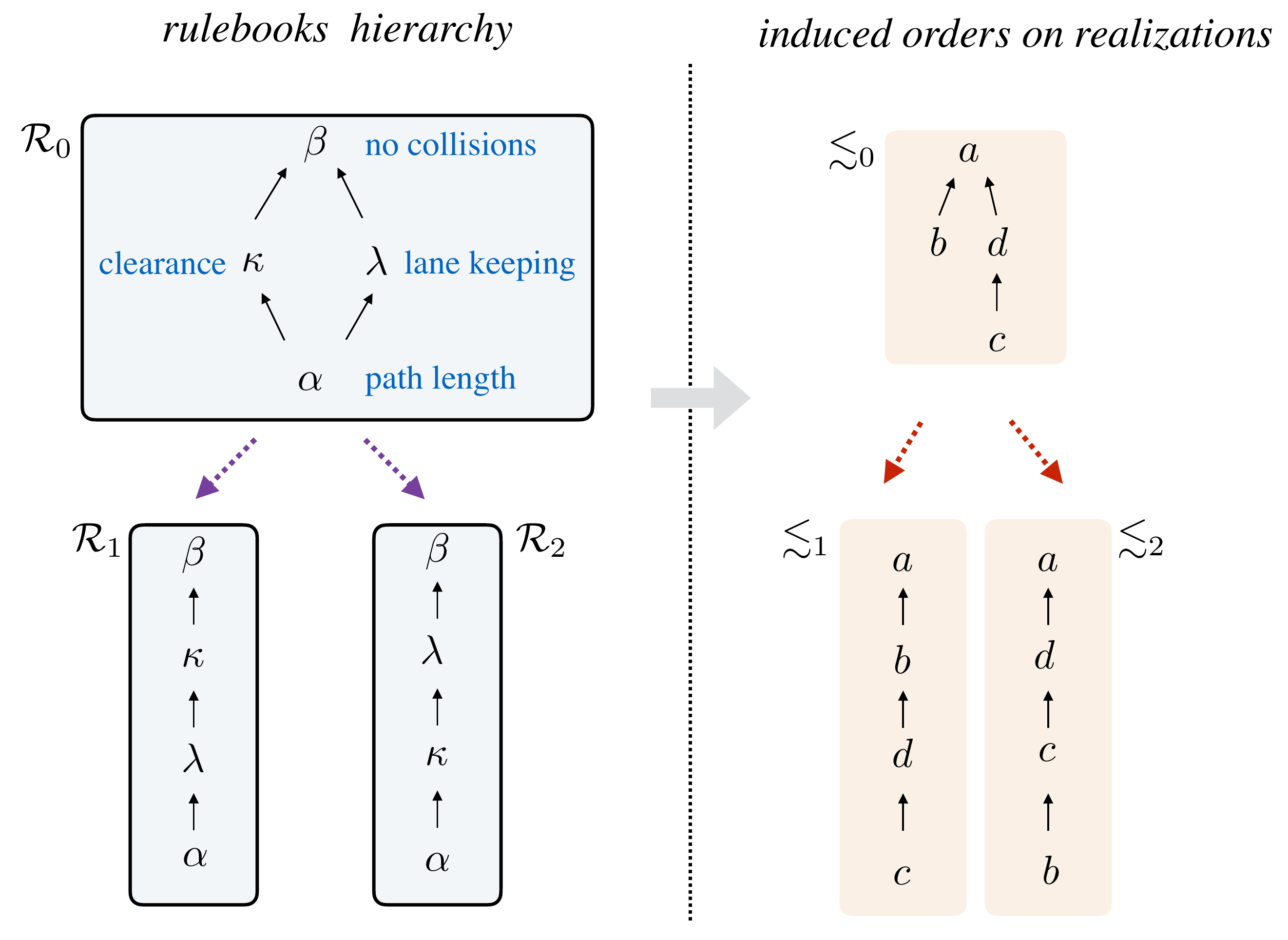}
\par\end{centering}
\centering{}}
\par\end{centering}
\caption{\label{fig:nontrivial}Example involving an avoidance maneuver.}
\end{figure}

\begin{center}
\par\end{center}

\section{Iterative specification refinement with rulebooks manipulation}

We formalize this process of \emph{iterative specification refinement~}(\prettyref{fig:Iterative-specification-refinmen})\emph{,}
by which a user can add rules and priority relations until the behavior
of the system is fully specified to one's desire.
\begin{example}
Regulations in different states and countries often share a great
deal of similarity. It would be ineffective to start the construction
of rulebooks from scratch in each case; rather, we wish to be able
to define a \textquotedbl base\textquotedbl{} rulebook that can then
be particularized for a specific legislation by adding rules or priority
relations.
\end{example}

\subsubsection{Operations that refine rulebooks}

We will consider three operations (\prettyref{fig:manipulation}):
\begin{enumerate}
\item \textbf{Priority refinement} (\prettyref{def:op-refine}): this operation
corresponds to adding another edge to the graph, thus clarifying the
priority relations between two rules. 
\item \textbf{Rule aggregation} (\prettyref{def:op-aggreg}): this operation
allows to ``collapse'' two or more equi-ranked rules into one.
\item \textbf{Rule augmentation }(\prettyref{def:op-augment}):\textbf{
}this operation consists in adding another rule at the lowest level
of priority.
\end{enumerate}
\begin{figure}[H]
\centering{}\includegraphics[width=1\columnwidth]{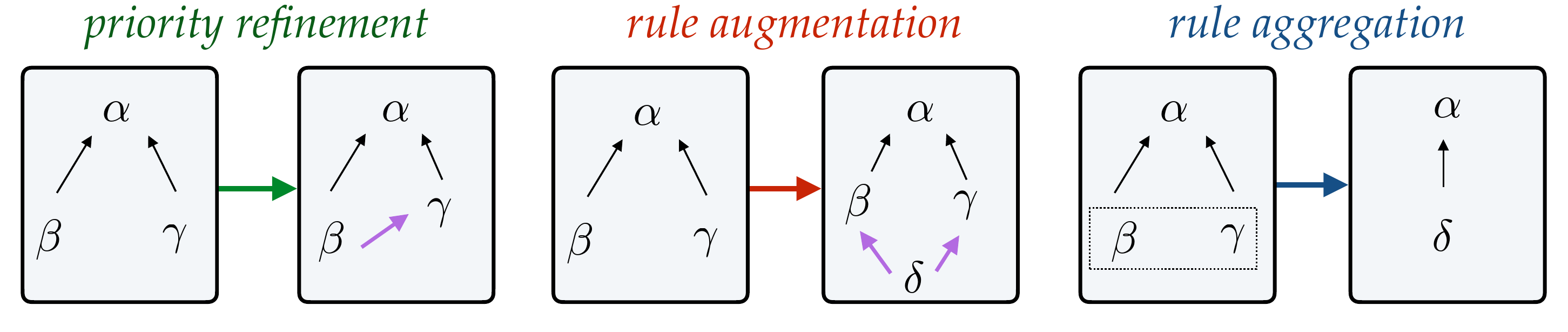}
\caption{\label{fig:manipulation} Three operations for manipulation of rulebooks.}
\end{figure}

\subsubsection{Priority refinement}

The operation of refinement adds priority constraints to the rulebook. 
\begin{defn}
\label{def:op-refine}An allowed \emph{priority refinement} operation
of a rulebook $\left\langle \mathcal{R}_{1},\rleq_{1}\right\rangle $
is a rulebook $\left\langle \mathcal{R}_{1},\rleq_{2}\right\rangle $,
where the order~$\rleq_{2}$ is a refinement of $\rleq_{1}$.
\end{defn}
\begin{example}
Continuing the example in~\prettyref{fig:nontrivial}, we can create
two refinements of the rulebook by adding priority constraints that
resolve the incomparability of rules~$\kappa$ and~$\lambda$ one
way or the other. For example, choosing the totally ordered rulebook
$\alpha\rightarrow\kappa\rightarrow\lambda\rightarrow\beta$, the
order on trajectories is $b<c<d<a$, while for the rulebook $\alpha\rightarrow\lambda\rightarrow\kappa\rightarrow\beta$,
the order is $c<d<b<a$.
\end{example}

\subsubsection{Rule aggregation}

Suppose that a rulebook includes two rules that are in the same equivalence
class. The minimal example is a rulebook~$\left\langle \mathcal{R},\rleq\right\rangle $
that has two rules~$r_{1},r_{2}$ such that~$r_{1}\rleq r_{2}$
and~$r_{2}\rleq r_{1}$. The induced order~$\les$ on the realizations
is that of the product order:
\[
x\les y\ \ \text{iff}\ \ r_{1}(x)\leq r_{1}(y)\ \wedge\ r_{2}(x)\leq r_{2}(y).
\]
We might ask whether we can ``aggregate'' the two rules into one.
The answer is positive, given the conditions in the following definition.
\begin{defn}[Rule aggregation operation]
\label{def:op-aggreg}Consider a rulebook~$\left\langle \mathcal{R},\rleq\right\rangle $
in which there are two rules~$r_{1},r_{2}\in\mathcal{R}$ that are
in the same equivalence class defined by~$\rleq$. Then it is allowed
to ``aggregate'' the two rules into a new rule~$r'$, defined by
\[
r'(x)=\alpha(r_{1}(x),r_{2}(x)),
\]
where~$\alpha$ is an embedding of the product pre-order into~$\mathbb{R}_{+}$. 

In particular, allowed choices for $\alpha$ include linear combinations
with positive coefficients ($\alpha(r_{1},r_{2})=a\,r_{1}+b\,r_{2}$)
and other functions that are strictly monotone in both arguments.
\end{defn}

\subsubsection{Rule augmentation}

Adding a rule to a rulebook is a potentially destructive operation.
In general, we can preserve the existing order only if the added rule
is below every other.
\begin{defn}[Rule augmentation]
\label{def:op-augment}The operation of rule augmentation consists
in adding to the rulebook~$\mathcal{R}$ a rule~$r'$ such that~$r'<r$
for all~$r\in\mathcal{{R}}.$
\end{defn}

\subsubsection{Properties preserved by the three operations}

We will show that the three operations create a rulebook that is a
refinement of the original rulebook, in the sense of \prettyref{def:rulebook-refinement-strict}. 

\begin{figure}[b]
\centering{}\includegraphics[width=1\columnwidth]{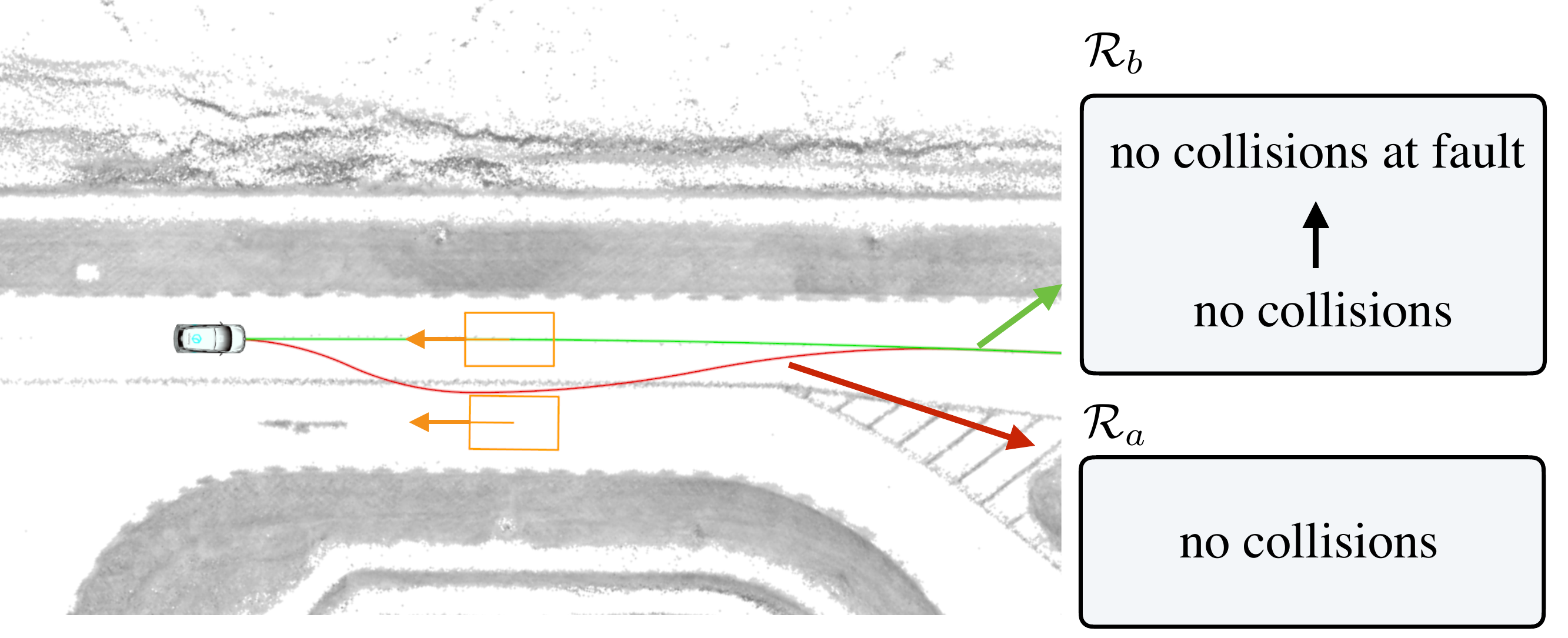}
\caption{\label{fig:experiment:collision}Trajectories planned in the unavoidable
collision scenario with different versions of the rulebooks. (See
attached videos for experiment.) The orange rectangles are the traffic
vehicles, moving towards the ego vehicle at the speed of 1.0 m/s.
The red trajectory is chosen when collision at fault and collision
caused by third-party are treated equally whereas the green trajectory
is chosen when collision at fault is higher in the rulebooks hierarchy
than the collision caused by third-party. }
 
\end{figure}

\begin{defn}
\label{def:rulebook-refinement-strict}A rulebook $\left\langle \mathcal{R}_{1},\rleq_{1}\right\rangle $
a\emph{ strict refinement }of $\left\langle \mathcal{R}_{2},\rleq_{2}\right\rangle $
if its induced strict pre-order~$\less_{2}$ refines~$\less_{1}$. 

\end{defn}
One can prove this theorem:
\begin{thm}
Applying one of the three operations (augmentation, refinement, aggregation)
to a rulebook~$\mathcal{R}_{1}$ creates a rulebook~$\mathcal{R}_{2}$
that is a strict refinement of~$\mathcal{R}_{1}$ in the sense of~\prettyref{def:rulebook-refinement-strict}.
\end{thm}

The proofs for these and ancillary results are in the appendix.%

\section{Experiments}

We show planning results for different rulebooks for the nuTonomy
R\&D platform (Renault Zoe). The experiments assume left-hand traffic
(Singapore/UK regulations). %

\subsubsection{Unavoidable collision}

This experiment illustrates unavoidable collision as described in~\exaref{liability}.
We set up the scenario~(\figref{experiment:collision}) such that
the planner is led to believe that 2 vehicles instantaneously appear
at approximately 12~m from the ego vehicle and slowly move towards
the ego vehicle at 1.0~m/s, while the speed of the ego vehicle is
9.5~m/s. \figref{experiment:collision}~shows the belief state when
the vehicles first appear. We also limit the allowed deceleration
to 3.5~$\text{m}/\text{s}^{2}$. It can be verified that collision
is unavoidable under these conditions.

For any given trajectory $x$, we define the collision cost as 
\begin{equation}
\mu(x)=v_{x,\text{col}},\label{eq:experiment:collision_cost}
\end{equation}
where $v_{x,\text{col}}$ is the expected longitudinal speed of the
ego vehicle at collision, assuming that the ego vehicle applies the
maximum deceleration from the current state.

First, consider the case where the collision cost~(\ref{eq:experiment:collision_cost})
is applied to any collision. In this case, it is more preferable to
swerve and hit the traffic vehicle in the opposite lane since the
swerving trajectory gives the ego vehicle more distance to decelerate;
hence, reducing the expected speed at collision.

Next, collision at fault~$\mu_{1}$ is differentiated from collision
caused by third-party~$\mu_{2}$ with priority~$\mu_{2}<\mu_{1}$.
The optimal trajectory in this case is to stay within lane and collide
with the traffic vehicle that is moving against the direction of traffic.

\subsubsection{Clearance and lane keeping}

In this experiment, we demonstrate how different rulebooks in~\exaref{partial}
lead to different behaviors when overtaking a stationary vehicle.
The blockage cost $\beta$, lane keeping cost $\lambda$ and length
$\alpha$ are defined as in~\exaref{partial}, but we re-define the
clearance cost as 
\begin{equation}
\kappa(x)=\max(0,C_{0}-l_{x}),
\end{equation}
where~$l_{x}$ is the minimum lateral distance between the stationary
vehicle and trajectory~$x$.%

In particular, we consider two different rulebooks (\figref{nontrivial}):
\begin{align}
\mathcal{R}_{1} & =\{\alpha<\lambda<\kappa<\beta\},\quad\text{(clearance first)}\label{eq:experiment:clearance_first}\\
\mathcal{R}_{2} & =\{\alpha<\kappa<\lambda<\beta\}.\quad\text{(lane keeping first)}\label{eq:experiment:lane_keeping_first}
\end{align}

The rulebook described in (\ref{eq:experiment:lane_keeping_first})
corresponds to the case where satisfying the lane keeping rule is
preferred over satisfying the clearance rule whereas the rulebook
described in (\ref{eq:experiment:clearance_first}) corresponds to
the case where satisfying the clearance rule is preferred over satisfying
the lane keeping rule.

\figref{experiment:clearance_vs_inlane} shows the optimal paths found
by the system in the two cases. With rulebook~$\mathcal{R}_{2}$,
the optimal trajectory is such that the vehicle footprint remains
within lane, leading to the violation of the clearance rule. In contrast,
when rulebook~$\mathcal{R}_{1}$ is applied, the trajectory crosses
the lane boundary to give sufficient clearance from the stationary
vehicle. 

\begin{figure}[h]
\centering{}\includegraphics[width=0.45\textwidth]{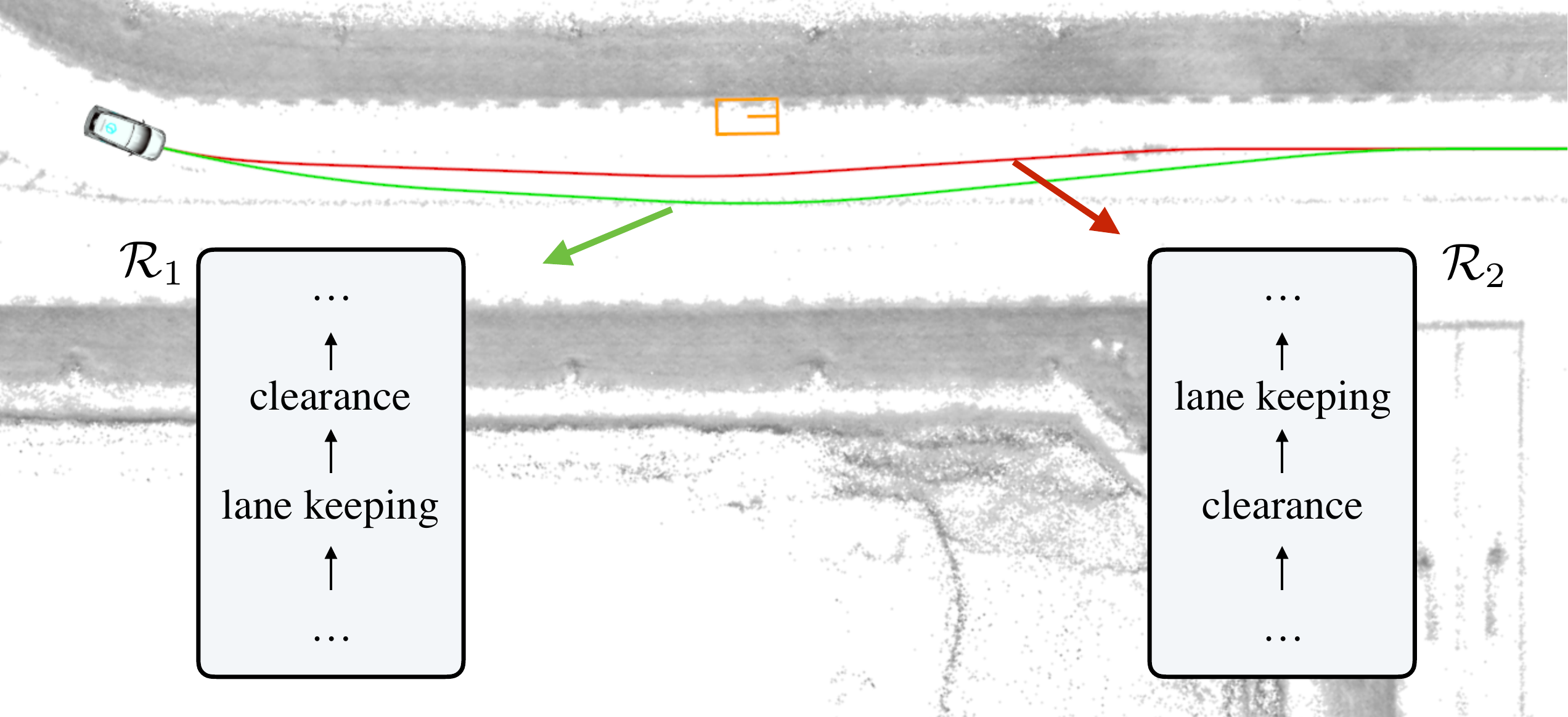}
\caption{\label{fig:experiment:clearance_vs_inlane}Trajectories planned in
the vehicle overtaking scenario with different rulebooks. (See attached
videos for experiment.) The orange rectangle is the stationary vehicle.
The red trajectory is when the rulebook (\ref{eq:experiment:lane_keeping_first})
is used, whereas the green trajectory is when the rulebook (\ref{eq:experiment:clearance_first})
is used. }
\end{figure}

\subsubsection{Lane change near intersection}

Consider the scenario where the autonomous vehicle needs to perform
a lane change in the vicinity of an intersection~(\figref{experiment:lane_change}).
The vehicle needs to turn left at the intersection; therefore, it
is required to be on the left lane before entering the intersection.
However, there is a stationary vehicle that prevents it from completing
the maneuver at an appropriate distance from the intersection.

For the simplicity of the presentation, we assume that any trajectory
$x$ only crosses the lane boundary once at $\eta_{x}$. The lane
change near intersection cost is then defined as $\zeta(x)=\max(0,D_{lc}-d_{\text{int}}(\eta_{x}))$,
where $D_{lc}$ is a predefined threshold of the distance from intersection,
beyond which changing lane is not penalized and for any pose $p$,
$d_{\text{int}}(p)$ is the distance from $p$ to the closest intersection.

Additionally, we define the turning cost $\tau(x)$ as the $L_{1}$-norm
of the heading difference between~$x$ and the nominal trajectory
associated with each lane. Consider the case where $\zeta$ and $\tau$
are in the same equivalence class and these rules are aggregated (\defref{op-aggreg})
as
\begin{equation}
r_{\zeta,\tau}(x)=\zeta(x)+c_{\tau}\tau(x),\label{eq:experiment:lane_change_int}
\end{equation}
where~$c_{\tau}>0$ is a predefined constant. In this experiment,
we consider the aggregated cost $r_{\zeta,\tau}$ and the blockage
cost $\beta$ defined in \exaref{partial} with priority~$r_{\zeta,\tau}<\beta$.
\figref{experiment:lane_change}~shows how the choice of~$c_{\tau}$
affects the optimal trajectory.

\begin{figure}[h]
\centering{}\includegraphics[width=1\columnwidth]{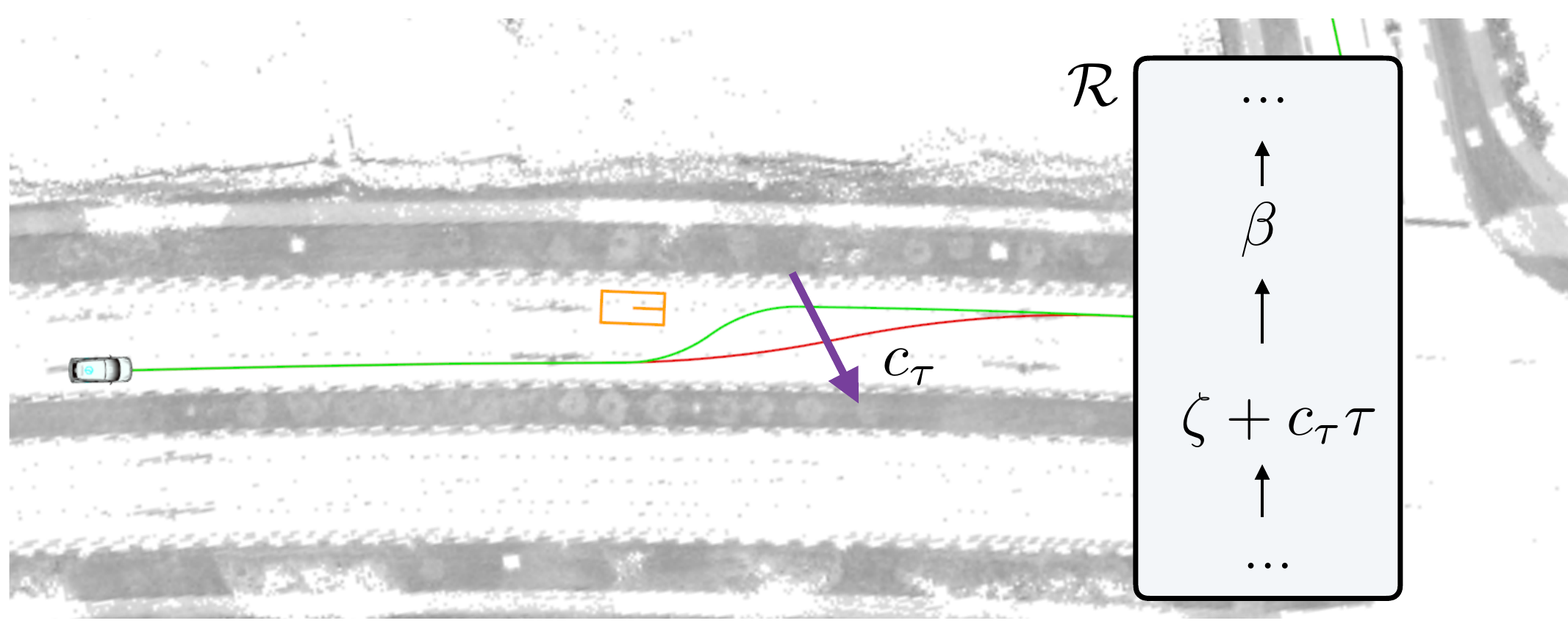}
\caption{\label{fig:experiment:lane_change}Trajectories planned in the lane
changing near intersection scenario. (See attached videos for experiment.)
The orange rectangle is the stationary vehicle at pose $p_{v}$ with
$d_{int}(p_{v})<D_{lc}$. The green trajectory is the optimal trajectory
for $c_{\tau}=0$ whereas the red trajectory is the optimal trajectory
for some $c_{\tau}>0$.}
\end{figure}

\section{Discussion and future work}

We have shown by way of a few examples how the rulebooks approach
allows easy and intuitive tuning of self-driving behavior. What is
difficult to convey in a short paper is the ability of the formalism
to scale up. In our production code at nuTonomy, corresponding to
level 4 autonomy in a limited operating domain, our rulebooks have
about 15 rules. For complete coverage of Massachusetts or Singapore
rules, including rare corner cases (such as ``do not scare farm animals''),
we estimate about 200 rules, to be organized in about a dozen ordered
priority groups~(\prettyref{fig:Macro-groups-for-rules}). 

Except the extrema of safety at the top, all the other priorities
among rule groups are somehow open for discussion. What we realized
is that some of the rules and rules priorities, especially those that
concern safety and liability, must \emph{be part of nation-wide and
global regulations} to be developed after an\emph{ informed public
discourse}; it should not be up to engineers to choose these important
aspects. The rulebooks formalism allows to have one such shared, high-level
specification that gives minimal constraints to the behavior; then,
the rest of the rules and priority choices can be considered ``implementation
details'' that might change from manufacturer to manufacturer.

\vfill

\begin{figure}[H]
\centering{}\includegraphics[height=9.5cm]{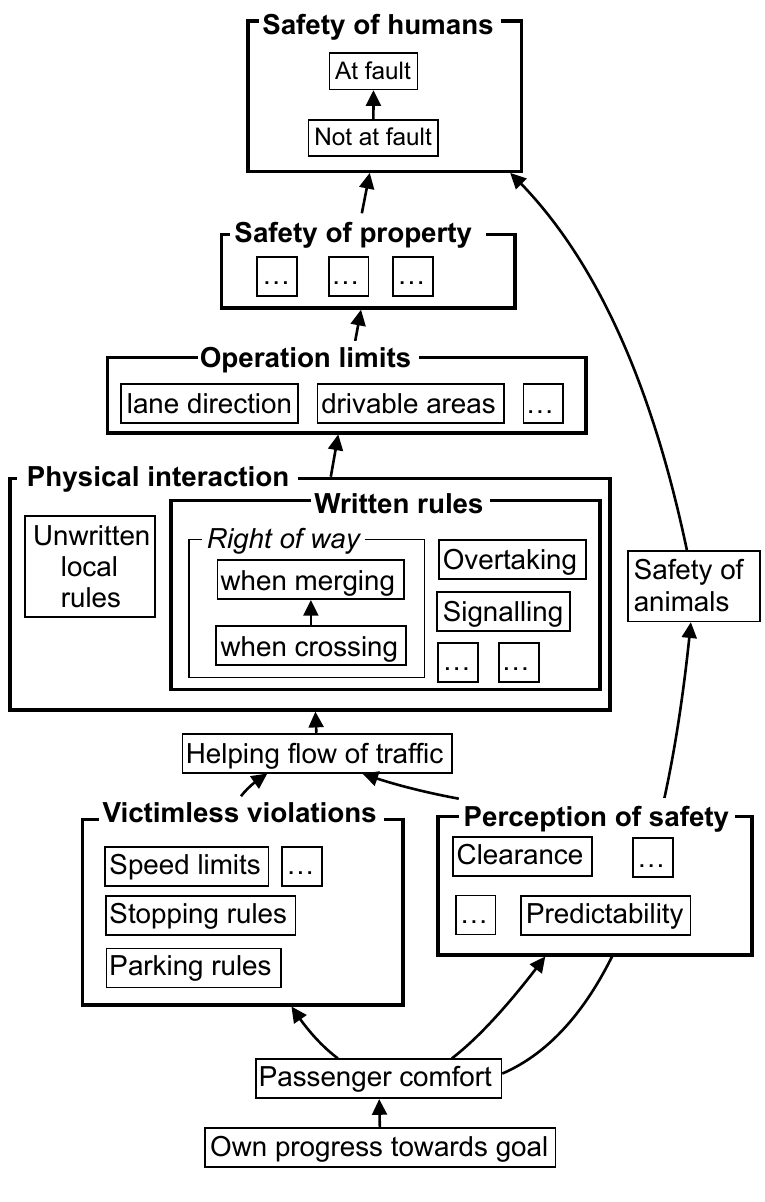}\caption{\label{fig:Macro-groups-for-rules}Illustrative example of possible
rule groups for an autonomous taxi in an urban driving scenario. At
the top of the hierarchy there are rules that guarantee safety of
humans; at the bottom, we have comfort constraints and progress goals.
At the top, the rules are written analytically; at the bottom, some
rules are learned from observed behavior. Rules at the bottom also
tend to be platform- and implementation- specific. Except for human
safety at the top, all other priorities among rule groups are open
for discussion. }
\end{figure}

\clearpage

\bibliographystyle{IEEEtran}
\bibliography{bibs}

\clearpage

\appendices

\section{Orders and Preorders}
\label{sec:orders-preorders}

This appendix recalls some standard definitions used in the development of the
rulebooks formalism.

\begin{defn}
  A \textit{preorder} on a set \( Z \) is a reflexive transitive binary relation
  \( \les \), i.e., a binary relation such that for all \( z \in Z \) one has \(
  z \les z \) and for all \( x, y, z \in Z \)
\[ x \les y \textrm{ and } y \les z \implies x \les z. \] Given a preorder \(
(Z, \les) \) and \( y, z \in Z \), the notation \(y < z\) is shorthand
for \(y \les z \) and \( z \not{\mkern -1mu}\les y \). A preorder is said to be
\textit{total} if additionally for any \( x, y \in Z \) either \( x \les y \) or
\( y \les x \).

With any preorder one associates an equivalence relation: elements
\( x, y \in Z \) are equivalent (denoted as \( x \sim y \)) whenever
\( x \les y \) and \( y \les x \).  If this equivalence relation is trivial
(i.e., \( x \sim y \) if and only if \( x = y \) ), then we say that \( \les \)
is a partial \textit{order} on \( Z \). We use \( x \le y \) to denote
(\( x < y \) or \( x = y \)). In particular, one always has
\( \le\, \subseteq\, \les \), and a preorder is an order if and only if
\( \le\, =\, \les \).
\end{defn}

\begin{defn}
  An \textit{embedding} between preorders \( Z_1 \) and \( Z_2 \) is a
  map \( \phi : Z_1 \to Z_2 \) such that for all \( x, y \in Z_1 \) one has
  \[ x \les y \implies \phi(x) \les \phi(y) \textrm{ and } x < y \implies
    \phi(x) < \phi(y). \]
\end{defn}

Note that for any embedding \( \phi \), equivalence \( x \sim y \) implies
\( \phi(x) \sim \phi(y) \).

\begin{defn}
  Given a set \( Z \) and two preorders \( \les_{1} \), \( \les_{2} \) on it, we
  say that \( \les_{2} \) \textit{refines} \( \les_{1} \) if the identity map
  \( \mathrm{id} : Z \to Z \) is an embedding from \( (Z, \les_{1}) \)
  onto \( (Z, \les_{2}) \).
\end{defn}

\section{Operations on Rulebooks}

\label{sec:compliance}

As explained in Section 2, a rulebook induces a partial preorder
\( \les \)~on realizations. We are interested in operations on
the rulebooks that preserve existing relations, but may possibly introduce new
comparisons between realizations, i.e., the preorder on realizations may be
refined.

More formally, our goal is to find conditions on a map
\( \phi : \mathcal{R}_{1} \to \mathcal{R}_{2} \) between two rulebooks that
guarantee that \( \les_{2} \) is a refinement of \( \les_{1} \). To motivate
concepts that will follow, let us begin with the simplest case of a rulebook
\( \mathcal{R}_{2} = \{ u \} \) consisting of a single rule. Let
\( \{r_{1}, \ldots, r_{n}\} = \mathcal{R}_{1} \) be the rules in the domain, and
the map \( \phi \) therefore collapses all \( r_{i} \) onto \( u \):
\( \phi(r_{i}) = u \) for all \( i \). At the moment we do not impose any
assumptions on \( r_{i} \) --- some of them may be comparable, some are
equivalent or independent. The question then becomes when the (total) preorder
imposed by \( u \) on \( \Xi \) is a refinement of the (partial) preorder given by
\( \{r_{1},\ldots, r_{n}\} \). Recall that according to the definition of a
refinement that amounts to:
\begin{multline}
  \label{eq:map-phi-refines}
  \forall x, y \in \Xi \quad \bigl(x \les_{1} y \implies x \les_{2} y\bigr)\quad
  \textrm{and} \\ \bigl(x <_{1} y \implies x <_{2} y\bigr).
\end{multline}

\subsection{Aggregative maps}
\label{sec:aggregative-maps}

The first observation is that \eqref{map-phi-refines} necessarily implies that
the value \( u(x) \) depends only on the values \( r_{1}(x), \ldots, r_{n}(x) \)
and not on the realization \( x \) itself.  Indeed, if \( x \) and \( y \) are
two realizations such that \( r_{i}(x) = r_{i}(y) \) for all \( i \), yet
\( u(x) \ne u(y) \) (say, \( u(x) < u(y) \) for definiteness), then
\( y \les_{1} x \), but \( y \not{\mkern-1mu}\les_{2} x \), contradicting
\eqref{map-phi-refines}.

One may view \( \mathcal{R}_{1} \) as providing a map from all realizations to
\( \bigl(\mathbb{R}^{+}\bigr)^{n} \) via
\[ x \mapsto \bigl(r_{1}(x), \ldots, r_{n}(x)\bigr), \]
and similarly \( \mathcal{R}_{2} \), consisting just of a single rule, can be
identified with a map \( u : \Xi \to \mathbb{R}^{+} \).  The observation above
can then be reinterpreted to say that there exists a map \( \alpha :
(\mathbb{R}^{+})^{n} \to \mathbb{R}^{+} \) making the following diagram
commutative:
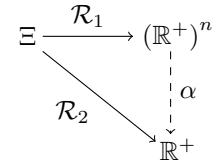
\begin{figure}[htb]
  \centering
  \begin{tikzpicture}
    \draw (0,1.5) node {\( \Xi \)};
    \draw (2,1.5) node {\( \bigl(\mathbb{R}^{+}\bigr)^{n} \)};
    \draw (2, 0) node {\( \mathbb{R}^{+} \)};
    \draw[->] (0.2, 1.5) -- (1.4, 1.5) node[pos=0.5, anchor=south] {\( \mathcal{R}_{1} \)};
    \draw[->, dashed] (1.9, 1.3) -- (1.9, 0.2) node[pos=0.5, anchor=west] {\( \alpha \)};
    \draw[->] (0.2, 1.3) -- (1.7, 0.1) node[pos=0.5, anchor=east, yshift=-2mm]
    {\( \mathcal{R}_{2} \)};
    \draw (1, 2.5) node {\( \forall x \in \Xi \quad \alpha\bigl(r_{1}(x), \ldots, r_{n}(x)\bigr) = u(x). \)};
  \end{tikzpicture}
  \caption{Factorization of \( u : \Xi \to \mathbb{R}^{+} \)}
  \label{fig:aggregative-diagram}
\end{figure}

What can be said about the map \( \alpha \) itself? The set \(
\bigl(\mathbb{R}^{+}\bigr)^{n} \) has a natural partial order on it, called the
product order, where given \( \vec{a}, \vec{b} \in
\bigl(\mathbb{R}^{+}\bigr)^{n} \), 
\[ \vec{a} = \bigl(a_{1}, \ldots, a_{n}\bigr), \quad \vec{b} = \bigl(b_{1},
  \ldots, b_{n}\bigr), \]
one denotes \( \vec{a} \le \vec{b} \) whenever \( a_{i} \le b_{i} \) for all
\( i \). Note that if \( x, y \in \Xi \) are two realizations such that
\[ \bigl(r_{1}(x), \ldots, r_{n}(x)\bigr) \le \bigl(r_{1}(y), \ldots,
  r_{n}(y)\bigr), \]
then necessarily \( x \les_{1} y \) and in view of \eqref{map-phi-refines} we
therefore need to have \( u(x) \le u(y) \). If moreover
\( r_{i}(x) < r_{i}(y) \) for at least some \( i \), then \( x <_{1} y \), and so
\( u(x) < u(y) \) must be the case. This observation can be summarized as
follows: If \( S \subseteq \bigl(\mathcal{R}^{+}\bigr)^{n} \), \( T \subseteq
\mathbb{R}^{+} \) are the sets
\[ S = \bigl\{ \bigl(r_{1}(x), \ldots, r_{n}(x)\bigr) : x \in \Xi \bigr\}, \quad T = \bigl\{u(x) : x
  \in \Xi\bigr\}, \]
then \( \alpha : S \to T \) is an embedding of partial orders.  This brings us to
the following

\begin{defn}[Aggregative map]
  We say that it is \textit{admissible to collapse rules} \( r_{1}, \ldots, r_{n} \) to a
  rule \( u \) if, in the notation above, the map \( \alpha \), that makes
  Figure~\ref{fig:aggregative-diagram} commutative, exists, and
  \( \alpha : S \to T \) is an embedding of partial orders.

  A map between rulebooks \( \phi : \mathcal{R} \to \mathcal{R}' \) is said to
  be \textit{aggregative} if for all \( u \in \phi(\mathcal{R}') \) it is
  admissible to collapse rules \( \phi^{-1}(u) \) onto \( u \).
\end{defn}

The following lemma shows that \textit{surjective} aggregative maps can be
composed yielding another surjective aggregative map.

\begin{lem}
  \label{lem:composition-aggregative-maps}
  Composition of surjective aggregative maps is aggregative.
\end{lem}

\begin{proof}
  Let \( \mathcal{R}, \mathcal{R}', \mathcal{R}'' \) be rulebooks, let
  \( \phi_{1} : \mathcal{R} \to \mathcal{R}' \),
  \( \phi_{2} : \mathcal{R}' \to \mathcal{R}'' \) be aggregative maps and let
  \( \psi : \mathcal{R} \to \mathcal{R}'' \) be the composition of the two,
  \( \psi = \phi_{2} \circ \phi_{1} \).  We need to show that \( \psi \) is
  aggregative and to this end pick some \( w \in \mathcal{R}'' \).  Let
  \( u_{1}, \ldots, u_{n} \) be the rules in \( \phi_{2}^{-1}(w) \), and for
  each \( 1 \le i \le n \) let \( r_{1}^{i}, \ldots, r_{m_{i}}^{i} \) be the
  rules enumerating \( \phi_{1}^{-1}(u_{i}) \) (see
  Figure~\ref{fig:preimage-of-w}).

  \begin{figure}[htb]
    \centering
    \begin{tikzpicture}
      \draw (0,3.5) node {\( r_{1}^{1}, \ldots, r_{m_{1}}^{1} \)};
      \draw (0,2.5) node {\( r_{1}^{2}, \ldots, r_{m_{2}}^{2} \)};
      \draw (0,1) node {\( r_{1}^{n-1}, \ldots, r_{m_{n-1}}^{n-1} \)};
      \draw (0,0) node {\( r_{1}^{n}, \ldots, r_{m_{n}}^{n} \)};
      \draw (3, 3.5) node {\( u_{1} \)};
      \draw[<-] (1, 3.5) -- (2.5, 3.5) node[pos=0.5, anchor=south] {\( \phi_{1}^{-1} \)};
      \draw (3, 2.5) node {\( u_{2} \)};
      \draw[<-] (1, 2.5) -- (2.5, 2.5);
      \draw (3, 1) node {\( u_{n-1} \)};
      \draw[<-] (1.3, 1) -- (2.5, 1);
      \draw (1.5, 1.85) node {\( \vdots \)};
      \draw (3, 0) node {\( u_{n} \)};
      \draw[<-] (1, 0) -- (2.5, 0);
      \draw (6, 1.75) node {\( w \)};
      \draw [decorate,decoration={brace,mirror,amplitude=10pt},xshift=-4pt,yshift=0pt]
      (3.5,-0.5) -- (3.5,4.0);
      \draw[<-] (4, 1.75) -- (5.7, 1.75) node[pos=0.5, anchor=south] {\( \phi_{2}^{-1} \)};
    \end{tikzpicture}
    \caption{Structure of the preimage \( \psi^{-1}(w) \) of \( w \)}
    \label{fig:preimage-of-w}
  \end{figure}
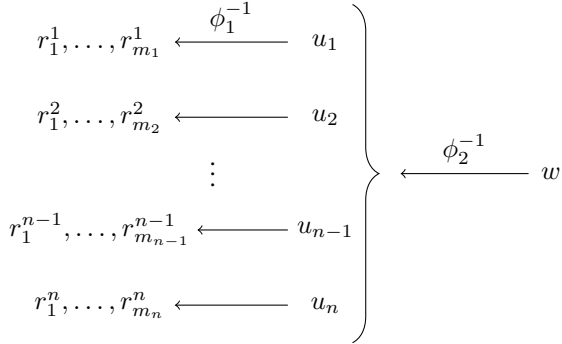

  The first observation is that for any realization \( x \in \Xi \), the value
  \( w(x) \) depends only on the numbers \( r_{j}^{i}(x) \).  Indeed, since
  \( \phi_{1} \) is aggregative, for each \( i \) the value \( u_{i}(x) \)
  depends only on \( \bigl(r_{j}^{i}(x)\bigr)_{1 \le j \le m_{i}} \), and by a
  similar token \( w(x) \) is uniquely reconstructible from \( u_{i}(x) \).
  More precisely, suppose
  \( \alpha_{i} : \bigl(\mathbb{R}^{+}\bigr)^{m_{i}} \to \mathbb{R}^{+} \) are
  the maps witnessing that \( \phi_{1} \) is aggregative, and
  \( \beta : \bigl(\mathbb{R}^{+}\bigr)^{n} \to \mathbb{R}^{+} \) is the corresponding map for
  \( \phi_{2} \).  If \( t = \sum_{i = 1}^{n}m_{i} \), then the map
  \( \gamma : \bigl(\mathbb{R}^{+}\bigr)^{t} \to \mathbb{R}^{+}\) given by
  \begin{multline*}
    \gamma(c_{1}, \ldots, c_{t}) \\
    = \beta\bigl(\alpha_{1}(c_{1}, \ldots, c_{m_{1}}),
    \alpha_{2}(c_{m_{1}+1}, \ldots, c_{m_{1} + m_{2}}), \ldots, \\
    \alpha_{n}(c_{t - m_{n}+1}, \ldots, c_{t})\bigr),
  \end{multline*}
  satisfies
  \begin{multline*}
    \gamma\bigl(r_{1}^{1}(x), \ldots, r_{m_{1}}^{1}(x), \ldots, r_{1}^{n}(x),
    \ldots, r_{m_{n}}^{n}(x) \bigr)\\ = \beta\bigl(u_{1}(x), \ldots,
    u_{n}(x)\bigr) = w(x)
  \end{multline*}
  for all \(x \in \Xi \).

  We need to show that \( \gamma \) is an embedding of partial orders, and to
  this end let \( x, y \in \Xi \) be realizations such that
  \( r_{j}^{i}(x) \le r_{j}^{i}(y) \) for all \( i, j \). We need to show
  that \( w(x) \le w(y) \). Note that since \( \alpha_{i} \)'s are embeddings,
  \begin{multline*}
    u_{i}(x) = \alpha_{i}\bigl(r_{1}^{i}(x), \ldots, r_{m_{i}}^{i}(x)\bigr) \le \\
    \alpha_{i}\bigl(r_{1}^{i}(y), \ldots, r_{m_{i}}^{i}(y)\bigr) = u_{i}(y).
  \end{multline*}

  Also, since \( \beta \) is an embedding, this implies that
  \begin{multline*}
    w(x) = \beta\bigl(u_{1}(x), \ldots, u_{n}(x)\bigr) \le \\
    \beta\bigl(u_{1}(y), \ldots, u_{n}(y)\bigr) = w(y),
  \end{multline*}
  and hence \( w(x) \le w(y) \) as claimed.

  Finally, if moreover \( r_{j}^{i}(x) < r_{j}^{i}(y) \) for some \( i, j \),
  then \( u_{i}(x) < u_{i}(y) \), and therefore also \( w(x) < w(y) \).  Thus
  \( \gamma \) is an embedding of partial orders, and therefore \( \psi \) is
  aggregative.
\end{proof}

We are now ready to introduce the key notion of an embedding between rulebooks.

\begin{defn}[Rulebook embedding]
  An \emph{embedding} between rulebooks is an aggregative map \( \phi : \mathcal{R} \to
  \mathcal{R}' \) that is also an embedding between \( \mathcal{R} \) and \(
  \mathcal{R}' \) as partially preordered sets.
\end{defn}

\begin{lem}
  \label{lem:surjective-embeddings}
  Let \( \mathcal{R}_{1} \) and \( \mathcal{R}_{2} \) be rulebooks, and let
  \( \phi : \mathcal{R}_{1} \to \mathcal{R}_{2} \) be a surjective embedding.
  Let \( x, y \in \Xi \) be realizations such that \( x \les_{1} y \).  If
  \( r \in \mathcal{R}_{1} \) is such that \( r(y) \ne r(x) \), then there
  exists \( u' \in \mathcal{R}_{2} \), such that 
  \[ u' \ge \phi(r) \textrm{ and } u'(x) < u'(y). \]
\end{lem}

\begin{proof}
  If \( r(y) < r(x) \), then there exists \( r_{1} > r \) such that
  \( r_{1}(x) < r_{1}(y) \).  If \( r(x) < r(y) \) to begin with, then we set
  \( r_{1} = r \).  In either case we have \( r_{1} \ge r \) and
  \( r_{1}(x) < r_{1}(y) \).  Set \( u_{1} = \phi(r_{1}) \).  We are done if
  \( u_{1}(x) < u_{1}(y) \).  Otherwise, let
  \( \{r_{1}^{1}, \ldots, r_{m_{1}}^{1}\} = \phi^{-1}(u_{1}) \) be the preimage
  of \( u_{1} \) (note that \( r_{1} \) is one of these elements)

  Since \( \phi \) is aggregative and since \( r_{1}(x) < r_{1}(y) \), there has
  to be some \( 1 \le i \le m_{1} \) such that
  \( r_{i}^{1}(y) < r_{i}^{1}(x) \).  Indeed, if \( r_{j}^{1}(x) \le
  r_{j}^{1}(y) \) for all \( j \), then 
  \[ \bigl(r_{1}^{1}(x), \ldots, r_{m_{1}}^{1}(x)\bigr) < \bigl(r_{1}^{1}(y),
    \ldots, r_{m_{1}}^{1}(y)\bigr) \]
  in the product order.  Hence, \( \phi \) being aggregative implies \( u_{1}(x)
  < u_{1}(y)\) contradicting our earlier assumption.  Thus \( r_{i}^{1}(y) <
  r_{i}^{1}(x) \) for some \( i \).

  In view of \( x \les_{1} y \), there exist \( r_{2} > r_{i}^{1} \) such that
  \( r_{2}(x) < r_{2}(y) \). Set \( u_{2} = \phi(r_{2}) \). Note that
  \[ r_{i}^{1} < r_{2} \implies \phi(r_{i}^{1}) < \phi(r_{2}) \iff u_{1} < u_{2}. \]
  We are done if \( u_{2}(x) < u_{2}(y) \).

  Suppose that \( u_{2}(x) \ge u_{2}(y) \) and let
  \[ \bigl\{r_{1}^{2}, \ldots, r_{m_{2}}^{2}\bigr\} = \phi^{-1}(u_{2}). \]
  By the same argument as above, there must exist some \( 1 \le i \le m_{2} \)
  such that \( r_{i}^{2}(x) > r_{i}^{2}(y) \).  In view of \( x \les_{1} y \),
  there is \( r_{3} > r_{i}^{2} \) such that \( r_{3} (x) < r_{3}(y)\).  Set
  \( u_{3} = \phi(r_{3})\).

  \begin{figure}[htb]
    \centering
    \begin{tikzpicture}
      \draw (0, 3) node [anchor = east] {\( u_{k} \)};
      \draw (-0.3,2.65) node {\( \vdots \)};
      \draw (0, 2) node[anchor = east] {\( u_{3} \)};
      \draw[->] (-0.3,1.2) -- (-0.3,1.8);
      \draw (0, 1) node[anchor = east] {\( u_{2} \)};
      \draw[->] (-0.3,0.2) -- (-0.3,0.8);
      \draw (0, 0) node[anchor = east] {\( u_{1} \)};
      \draw (2,0) node {\( r_{1}^{1}, \ldots, r_{m_{1}}^{1} \)};
      \draw[<-] (0,0) -- (1,0) node[pos=0.5, anchor=north] {\( \phi \)};
      \draw (2,1) node {\( r_{1}^{2}, \ldots, r_{m_{2}}^{2} \)};
      \draw[<-] (0,1) -- (1,1);
      \draw (2,2) node {\( r_{1}^{3}, \ldots, r_{m_{3}}^{3} \)};
      \draw[<-] (0,2) -- (1,2);
      \draw (2,3) node {\( r_{1}^{k}, \ldots, r_{m_{k}}^{k} \)};
      \draw[<-] (0,3) -- (1,3);
    \end{tikzpicture}
    \caption{Construction of the chain }
    \label{fig:chain-construction}
  \end{figure}
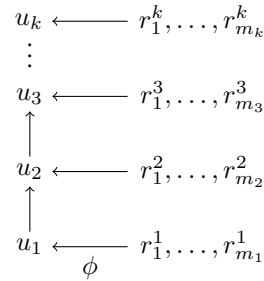

  The process continues, and builds a sequence of rules
  \( u_{1} < u_{2} < \cdots < u_{k} \) as in
  Figure~\ref{fig:chain-construction}.  By finiteness of the rulebook, the chain
  has to stop at some point, which is possible only if
  \( u_{k}(x) < u_{k}(y) \).  Since \( u_{k} > u_{1} \ge \phi(r) \),
  the lemma follows.
\end{proof}

\begin{thm}
  \label{thm:surjective-morphisms-preserve-order}
  Let \( \mathcal{R}_{1} \) and \( \mathcal{R}_{2} \) be rulebooks.  If there
  exists a surjective embedding \( \phi : \mathcal{R}_{1} \to \mathcal{R}_{2} \)
  between the two, then \( \les_{2} \) refines \( \les_{1} \).
\end{thm}

\begin{proof}
  Suppose \( x, y \in \Xi \) are such that \( x \les_{1} y \), we show that \( x
  \les_{2} y \).  Pick some \( u \in \mathcal{R}_{2} \) such that \( u(y) < u(x)
  \).  Let \( \{r_{1}, \ldots, r_{m}\} = \phi^{-1}(u) \) be the preimage.  Note
  that \( r_{i}(x) > r_{i}(y) \) for some \( i \) (for otherwise \( u(x) \le
  u(y) \), because \( \phi \) is aggregative), hence Lemma
  \ref{lem:surjective-embeddings} applies and produces some \( u' \ge
  \phi(r_{i}) = u \) such
  that \( u'(x) < u'(y) \).  

  It remains to show that \( x <_{1} y \) implies \( x <_{2} y \).  Since
  \( x \les_{2} y \) has already been shown,
  it is enough to show that \( u(x) \ne u(y) \) for some
  \( u \in \mathcal{R}_{2} \).  Pick some \( r \in \mathcal{R}_{1} \) such that
  \( r(x) < r(y) \).  Lemma \ref{lem:surjective-embeddings} produces
  \( u \ge \phi(r) \) such that \( u(x) < u(y) \).
\end{proof}

Two examples of surjective embeddings are Priority
Refinements~(\prettyref{def:op-refine}) and Rule
Aggregation~(\prettyref{def:op-aggreg}).

\section{Adding new rules}
\label{sec:adding-new-rules}

There is one important operation that is missing from the picture --- addition
of new rules.  Surjective embeddings of rulebooks let us impose new relations
between existing rules, as well as to aggregate several rules into one.  What if
one would like to add a new rule that does not bear any direct relation to
existing ones?

Generally, this is a very destructive operation in the sense that it can
dramatically change the preorder imposed on realizations.  Perhaps the most
extreme example is when to a rulebook \( \mathcal{R}_1 \) a new rule \( r \) is
added that is declared to be of the highest importance: \( u < r \) for all
\( u \in \mathcal{R}_1 \).  Let \( \mathcal{R}_2 \) denote the resulting
rulebook \( \mathcal{R}_1 \cup \{ r \} \), and note that if \( x, y \) are two
realizations such that \( r(x) < r(y) \) then necessarily \( x <_2 y\)
regardless of how they were related in the preorder induced by
\( \mathcal{R}_1 \).

A similar but slightly more general case is when a new rule is added ``in the
middle'' of \( \mathcal{R}_1 \).  More formally, suppose that
\( \mathcal{R}_2 = \mathcal{R}_1 \cup \{ r \} \), where the new rule \( r \)
satisfies \( u < r\) for some \( u \in \mathcal{R}_1 \).  If \( x, y \in \Xi \)
are two realizations such that \( u'(x) = u'(y) \) for all
\( u' \in \mathcal{R}_1\setminus \{u\} \) and \( u(x) < u(y) \), then
\( x <_1 y \).  If, however, \( r(y) < r(x) \), then \( y <_2 x \), and the
order between the realizations is reverted.

Unless one makes some additional assumptions on the set of realizations
\( \Xi \), the example above shows that adding a rule above an existing one can
easily change the preorder on realizations.  However, in order to get some
meaningful preservation of the preorder, it is not enough to assume that the
newly added rules are not above any of the existing ones. Consider the simplest
case, when \( \mathcal{R}_1 \) consists of a single rule \( \{ u \} \), and
\( \mathcal{R}_2 = \{r, u\} \) adds a rule that is incomparable with \( r \).
If \( x \) and \( y \) are two realizations such that \( u(x) < u(y) \), then
necessarily \( x <_1 y \).  However, if \( r(y) < r(x) \), then \( x \) and
\( y \) are incomparable relative to \( \les_2 \).  When the two rules are
further aggregated as described in the previous appendix, all relations between
\( x \) and \( y \) become possible.

The example above can be modified slightly by considering a rulebook
\( \mathcal{R}_1 = \{u, u'\} \), \( u < u' \), and adding the rule \( r \) such
that \( r < u' \), but \( r \) and \( u \) are incomparable.  The same analysis
as above now applies to a pair of realizations such that \( u'(x) = u'(y) \).
In particular, for the relation \( \les \) to
be broken by adding a new rule, one does not have to add a completely independent
rule, it is enough to have some rules in \( \mathcal{R}_1 \) that are not
comparable to \( r \).

We are left with only one option --- add new rules below all of the existing
ones.  However, even this operation does not result in the refinement of the
\( \les \) order on realizations.  Indeed, the simplest case, is when
\( \mathcal{R}_1 = \{u\} \) and \( \mathcal{R}_2 = \{u, r\} \), \( r < u \).  If
\( x, y \) are two equivalent realizations, then necessarily \( x \les_1 y \)
and \( y \les_1 x \).  However, if the new rule \( r \) differentiates between
the realizations, \( r(x) \ne r(y) \), then one of \( x \les_2 y \),
\( y \les_2 x \) is false.

We conclude that in general we cannot guarantee that the relation \( \les \) has
been refined if any new rules were added.  The last example in the list above
is, nonetheless, different from others.  It turns out that the only problem that
can occur, when new rules are added below existing ones, is that equivalent
realizations are no longer equivalent in the enlarged rulebook.  Thus, while
the preorder \( \les \) may not be refined, its strict counterpart \( < \) is
preserved by such an operation.

\begin{defn}
  An embedding of rulebooks \( \phi: \mathcal{R}_1 \to \mathcal{R}_2 \) is said
  to be \textit{dominant}, if \( u < r \) for all
  \( u \in \mathcal{R}_2 \setminus \phi(\mathcal{R}_1) \) and
  \( r \in \phi(\mathcal{R}_1) \).
\end{defn}

\begin{thm}
  Let \( \phi : \mathcal{R}_1 \to \mathcal{R}_2 \) be a dominant embedding of
  rulebooks.  If \( x, y \) are realizations such that \( x <_1 y \), then also
  \( x <_2 y\).
\end{thm}
\begin{proof}
  Consider \( \phi(\mathcal{R}_1) \) as a rulebook, and note that the map
  \( \phi : \mathcal{R}_1 \to \phi(\mathcal{R}_1) \) is automatically
  surjective.  Theorem~\ref{thm:surjective-morphisms-preserve-order} applies,
  and shows that \( x < y \) relative to \( \phi(\mathcal{R}_1) \) as well.
  This allows us to assume without loss of generality that
  \( \mathcal{R}_1 \subseteq \mathcal{R}_2 \), and the map \( \phi \) is the
  identity map.  Since \( \phi \) is assumed to be dominant, it means that
  \( u < r\) for all \( r \in \mathcal{R}_1 \) and
  \( u \in \mathcal{R}_2 \setminus \mathcal{R}_1 \).

  Suppose \( x, y \) are two realizations such that \( x <_1 y \), and let
  \( r \in \mathcal{R}_2 \) be such that \( r(y) < r(x) \).  We need to show
  that there exists some rule \( r' > r \) such that \( r'(x) < r'(y) \).
  Indeed, if \( r \in \mathcal{R}_1 \), then such a rule \( r' \) must exist
  simply because \( x <_1 y \) by assumption.  So, let us assume that
  \( r \in \mathcal{R}_2 \setminus \mathcal{R}_1 \).  Since \( x <_1 y \) there
  mush be at least one rule \( u \in \mathcal{R}_1 \) such that
  \( u(x) < u(y) \).  Since \( u > r \), the theorem follows.
\end{proof}

Rule Augmentation as described in \prettyref{def:op-augment} is an example
of a dominant embedding.

\end{document}